\documentclass[twoside]{article}

\usepackage[accepted]{aistats2017}
\usepackage{hyperref}       
\usepackage{url}            
\usepackage{booktabs}       
\usepackage{amsfonts}       
\usepackage{microtype}      
\usepackage{algorithm}
\usepackage{algorithmic}

\usepackage{graphicx}
\usepackage{caption}
\usepackage{subcaption}
\usepackage{xspace}
\usepackage{color}

\usepackage{amsmath}
\usepackage{amsthm}
\usepackage{natbib}
\setcitestyle{authoryear,open={(},close={)}}
\usepackage{bbm}

\usepackage[titletoc,title]{appendix}
\allowdisplaybreaks

\newcommand{\ra}[1]{\renewcommand{\arraystretch}{#1}}

\DeclareMathOperator*{\argmin}{arg\,min}

\makeatletter
\newtheorem*{rep@theorem}{\rep@title}
\newcommand{\newreptheorem}[2]{%
\newenvironment{rep#1}[1]{%
 \def\rep@title{#2 \ref{##1}}%
 \begin{rep@theorem}}%
 {\end{rep@theorem}}}
\makeatother

\newcommand{\innerprod}[3][\defmu]{{\left\langle {#2},{#3} \right\rangle}_{#1}}

\newreptheorem{theorem}{Theorem}
\newreptheorem{definition}{Definition}

\newtheorem{theorem}{Theorem}[section]
\newtheorem{lemma}[theorem]{Lemma}
\newtheorem{proposition}[theorem]{Proposition}

\newtheorem{definition}[theorem]{Definition}

\newcommand{\algname}{Streaming Gradient Boosting\xspace}
\newcommand{\algshort}{SGB\xspace}

\begin{document}

%

%

\twocolumn[

\aistatstitle{Gradient Boosting on Stochastic Data Streams}

\aistatsauthor{Hanzhang Hu \And Wen Sun \And Arun Venkatraman \And Martial Hebert \And J. Andrew Bagnell}

\aistatsaddress{School of Computer Science, Carnegie Mellon University \\\{hanzhang, wensun, arunvenk, hebert, dbagnell\}@cs.cmu.edu } 
]


\begin{abstract}
Boosting is a popular ensemble algorithm that generates more powerful learners by linearly combining base models from a simpler hypothesis class.  In this work, we investigate the problem of adapting batch gradient boosting for minimizing convex loss functions to online setting where the loss at each iteration is i.i.d sampled from an unknown distribution. To generalize from batch to online, we first introduce the definition of online weak learning edge with which for strongly convex and smooth loss functions, we present an algorithm, \algname (\algshort) with exponential shrinkage guarantees in the number of weak learners.  
We further present an adaptation of \algshort to optimize non-smooth loss functions, for which we derive a $O(\ln N / N)$ convergence rate. We also show that our analysis can extend to adversarial online learning setting under a stronger assumption that the online weak learning edge will hold in adversarial setting.
We finally demonstrate experimental results showing that in practice our algorithms can achieve competitive results as classic gradient boosting while using less computation.
\end{abstract}

\section{INTRODUCTION}
Boosting \citep{freund1995desicion} is a popular method that leverages simple learning models (e.g., decision stumps) to generate powerful learners. Boosting has been used to great effect and trump other learning algorithms in a variety of applications. In computer vision, boosting was made popular by the seminal Viola-Jones Cascade~\citep{viola2001rapid} and is still used to generate state-of-the-art results in pedestrian detection~\citep{nam2014local,yang2015convolutional,zhu2016}. Boosting has also found success in domains ranging from document relevance ranking~\citep{yahoo} and transportation~\citep{zhang2015gradient} to medical inference~\citep{atkinson2012assessing}. Finally, boosting yields an anytime property at test time, which allows it to work with varying computation budgets~\citep{grubb2012speedboost} for use in real-time applications such as controls and robotics.

The advent of large-scale data-sets has driven the need for adapting boosting from the traditional batch setting, where the optimization is done over the whole dataset, to the online setting where the weak learners (models) can be updated with streaming data. 
In fact, online boosting has received tremendous attention so far. For classification, \citep{lu:2012,Oza01onlinebagging,beygelzimer2015optimal} proposed online boosting algorithms along with theoretical justifications. Recent work by ~\cite{beygelzimer2015online}, addressed the regression task through the introduction of \emph{Online Gradient Boosting} (OGB). 
We build upon on the developments in~\citep{beygelzimer2015online} to devise a new set of algorithms presented below. 

In this work, we develop streaming boosting algorithms for regression with strong theoretical guarantees under stochastic setting, where at each round the data are i.i.d sampled from some unknown fixed distribution. 
In particular, our algorithms are streaming extension to the classic gradient boosting~\citep{friedman2001greedy}, where weak predictors are trained in a stage-wise fashion to approximate the functional gradient of the loss with respect to the previous ensemble prediction, a procedure that is shown by~\cite{Mason00boostingalgorithms} to be functional gradient descent of the loss in the space of predictors. Since the weak learners cannot match the gradients of the loss exactly, we measure the error of approximation by redefining of \emph{edge} of online weak learners~\citep{beygelzimer2015optimal} for online regression setting. 

Assuming a non-trivial edge can be achieved by each deployed weak online learner, we develop algorithms to handle smooth or non-smooth loss functions, and theoretically analyze the convergence rates of our streaming boosting algorithms. Our first algorithm targets strongly convex and smooth loss functions and achieves exponential decay on the average regret with respect to the number of weak learners. We show the ratio of the decay depends on the edge and also the condition number of the loss function. The second algorithm, designed for strongly convex but non-smooth loss functions, extends from the batch residual gradient boosting algorithm from \citep{grubb2011generalized}. We show that the algorithm achieves $O(\ln N/N)$ convergence rate with respect to the number of weak learners $N$, which matches the online gradient descent (OGD)'s no-regret rate for strongly convex loss \citep{hazan2007logarithmic}.  Both of our algorithms promise that as $T$ (the number of samples) and $N$ go to infinity, the average regret converges to zero. Our analysis leverages Online-to-Batch reduction \citep{cesa2004generalization,JMLR:v15:hazan14a}, hence our results naturally extends to adversarial online learning setting as long as the weak online learning edge holds in adversarial setting, a harsher setting than stochastic setting. We conclude with some proof-of-concept experiments to support our analysis. We demonstrate that our algorithm significantly boosts the performance of weak learners and converges to the performance of classic gradient boosting with less computation.

\section{RELATED WORK}

Online boosting algorithms have been evolving since their batch counterparts are introduced. \cite{Oza01onlinebagging} developed some of the first online boosting algorithm, and their work are applied to online feature selection~\citep{grabner:2006} and online semi-supervised learning~\citep{grabner:2008}. 
\cite{leistner:2009} introduced online gradient boosting for the classification setting albeit without a theoretical analysis. \cite{lu:2012} developed the first convergence guarantees of online boosting for classification. Then ~\cite{beygelzimer2015optimal} presented two online classification boosting algorithms that are proved to be respectively optimal and adaptive. 

Our work is most related to \citep{beygelzimer2015online}, which extends gradient boosting for regression to the online setting under a smooth loss: each weak online learner is trained by minimizing a linear loss, and weak learners are combined using Frank-Wolfe~\citep{frank1956algorithm} fashioned updates. Their analysis generalizes those of batch boosting for regression~\citep{zhang:2005}. In particular, these proofs forgo edge assumptions of the weak learners. 
Though Frank-Wolfe is a nice projection-free algorithm, it has relatively slow convergence and usually is restricted to smooth loss functions. 
In our work, each weak learner instead minimizes the squared loss between its prediction and the gradient, which allows us to treat weak learners as approximations of the gradients thanks to the weak learner edge assumption. Hence we can mimic classic gradient boosting and use a gradient descent approach to combine the weak learners' predictions.
These differences enable our algorithms to handle non-smooth convex losses, such as hinge and $L_1$-losses, and result in convergence bounds that is more analogous to the bounds of classic batch boosting algorithms. 
This work also differs from \citep{beygelzimer2015online} in that we assume an online weak learner edge exists, a common assumption in the classic boosting literature~\citep{freund1995desicion, freund1999intro2boost} that is extended to the online boosting for classification by~\citep{lu:2012,beygelzimer2015optimal}. With this assumption, we analyze online gradient boosting using techniques from gradient descent for convex losses~\citep{hazan2007logarithmic}.

\section{PRELIMINARIES}
In the classic online learning setting, at every time step $t$, the learner $\mathcal{A}$ first makes a prediction (i.e., picks a predictor $f_t \in\mathcal{F}$, where $\mathcal{F}$ is a pre-defined class of predictors) on the input  $x_t\in\mathbb{R}^d$, then receives a loss $\ell_t(f_t(x_t))$. The learner then updates $f_t$ to $f_{t+1}$. The samples $(\ell_t,x_t)$ could be generated by an adversary, but this work mainly focuses on the setting where $(\ell_t,x_t)\sim D$ are i.i.d sampled from a distribution $D$.
The regret $R_{\mathcal{A}}(T)$ of the learner is defined as the difference between the total loss from the learner and the total loss from the best hypothesis in hindsight under the sequence of samples  $\{ (\ell_t, x_t) \}_t$:
\begin{align}
    R_{\mathcal{A}}(T) = \sum_{t=1}^T \ell_t(f_t(x_t)) - \min_{f^*\in\mathcal{F}}\sum_{t=1}^T\ell_t(f^*(x_t)).
\end{align} 
We say the online learner is \emph{no-regret} if and only if $R_{\mathcal{A}}(T)$ is $o(T)$. That is, time averaged, the online learner predictor $f_t$ is doing as well as the best hypothesis $f^*$ in hindsight. 
We define \emph{risk} of a hypothesis $f$ as $\mathbb{E}_{(\ell, x)\sim D}[\ell(f(x))]$. Our analysis of the risk leverages the classic Online-to-Batch reduction \citep{cesa2004generalization,JMLR:v15:hazan14a}. The online-to-batch reduction first analyzes  regret without the stochastic assumption on the sequence of loss $\ell$, and it then relates  regret to  risk using concentration of measure. 

Throughout the paper we will use the concepts of strong convexity and smoothness.
A function $\ell(x)$ is said to be $\lambda$-strongly convex and $\beta$-smooth with respect to norm $\|\cdot\|$ if and only if for any pair $x_1$ and $x_2$:
\begin{align}
 &\frac{\lambda}{2}\|x_1 - x_{2}\|^2 \leq \ell(x_1) - \ell(x_{2}) - \nabla \ell(x_2)(x_1-x_{2})  \nonumber\\
 &\leq \frac{\beta}{2}\|x_1 - x_{2}\|^2,
\end{align} where $\nabla \ell(x)$ denotes the gradient of function $\ell$ with respect to $x$. 
å
\subsection{Online Boosting Setup}
\label{sec:boosting_setup}
Our online boosting setup is similar to~\citep{beygelzimer2015optimal} and~\citep{beygelzimer2015online}. At each time step $t = 1,..,T$, the environment picks loss $\ell_t: \mathbb{R}^m \rightarrow \mathbb{R}$. The online boosting learner makes a prediction ${y}_t \in \mathbb{R}^m$ without knowing $\ell_t$. Then the learner suffers loss $\ell_t({y}_t)$. Throughout the paper we assume the loss is bounded as $|\ell_t(y)|\leq B, B\in\mathbb{R}^+, \forall t,y$. We also assume that the gradient of the loss $\nabla\ell_t(y)$ is also bounded as $\|\nabla\ell_t(y)\|\leq G, G\in\mathbb{R}^+, \forall t,y$.\footnote{Throughout the paper, the notation $\|x\|$ for any finite dimension vector $x$ stands for the classic L2 norm.} The online boosting learner maintains a sequence of weak online learning algorithms $\mathcal{A}_1, ..., \mathcal{A}_N$. 
Each weak learner $\mathcal{A}_i$ can only use hypothesis from a restricted hypothesis class $\mathcal{H}$ to produce its prediction $\hat{y}^i_t = h_t^i(x_t)$ ($h: \mathbb{R}^d\to\mathbb{R}^m,\forall h\in\mathcal{H}$), where $h_t^i\in\mathcal{H}$. 
To make a prediction ${y}_t$ at each iteration, each $\mathcal{A}_i$ will first make a prediction $\hat{y}^i_t\in\mathbb{R}^m$ where $\hat{y}^i_t = h_t^i(x_t)$. The online boosting learner combines all the weak learners' predictions to produce the final prediction $y_t$ for sample $x_t$. 
The online learner then suffers loss $\ell_t(y_t)$ after the loss $\ell_t$ is revealed. As we will show later, with the loss $\ell_t$, the online learner will pass a square loss to each weak learner. Each weak learner will then use its internal no-regret online update procedure to update its own weak hypothesis from $h_t^i$ to $h_{t+1}^i$.  
In stochastic setting where $\ell_t$ and $x_t$ are i.i.d samples from a fixed distribution, the online boosting learner will output a combination of the hypothesises that were generated by weak learners as the final boosted hypothesis for future testing.

By leveraging linear combination of weak learners, the goal of the online boosting learner is to boost the performance of a single online learner $\mathcal{A}_i$. 
Additionally, we ideally want the prediction error to decrease exponentially fast in the number $N$ of weak learners, as is the result from classic batch gradient boosting \citep{grubb2011generalized}. 

\section{WEAK ONLINE LEARNING}
\label{sec:weak_learn}
We specifically consider the setting where each weak learner minimizes a square loss \mbox{$\|y - h(x)\|^2$}, where  $y$ is the regression target, and $h$ is in the weak-learner hypothesis class $\mathcal{H}$.
At each step $t$, a weak online learner $\mathcal{A}$ chooses a predictor $h_t \in \mathcal{H}$ to predict $h_t(x_t)$, receives the target $y_t$\footnote{Abuse of notation: in Sec~\ref{sec:weak_learn}, $y_t\in\mathbb{R}^m$ simply stands for a regression target for the weak learner at step $t$, not the final prediction of the boosted learner defined in Sec.~\ref{sec:boosting_setup}.} and then suffers loss $\|y_t - h_t(x_t)\|^2$. With this, we now introduce the definition of Weak Online Learning Edge.
\begin{definition}
\label{def:weak_learning}(\textbf{Weak Online Learning Edge}) Given a restricted hypothesis class $\mathcal{H}$ and a sequence of square losses $\{\|y_t - h(x_t)\|^2\}_t$, the weak online learner predicts a sequence $\{h_t\}$ that has edge $\gamma \in (0,1]$, such that with high probability $1-\delta$:
\begin{align}
\label{eq:weak_learner_eq}
\sum_{t=1}^T \|y_t - h_t(x_t)\|^2\leq (1-\gamma)\sum_{t=1}^T \|y_t\|^2 + R(T),
\end{align} where $R(T)\in o(T)$ is usually known as the excess loss. 
\end{definition} The high probability $1-\delta$ comes from the possible randomness of the weak online learner and the sequence of examples. Usually the dependence of the high probability bound on $\delta$ is poly-logarithmic in $1/\delta$ that is included in the term $R(T)$. We will give a concrete example on this edge definition in next section where we will show what $R(T)$ consists of.  Intuitively, a larger edge implies that the hypothesis is able to better explain the variance of the learning targets $y$. Our online weak learning definition is closely related to the one from~\citep{beygelzimer2015optimal} in that our definition is an result of the following two assumptions: (1) the online learning problem is agnostic-learnable (i.e., the weak learner has $\frac{o(T)}{T} \rightarrow 0$ time-averaged regret against the best hypothesis $h \in \mathcal{H}$) with high probability:
\begin{align}
\sum_{t=1}^T \|y_t - h_t(x_t)\|^2\leq \min_{h\in\mathcal{H}}\sum_{t=1}^T \|y_t-h(x_t)\|^2 + o(T),
\label{eq:learnable}
\end{align} 
and (2) the restricted hypothesis class $\mathcal{H}$ is rich enough such that for any sequence of $\{y_t, x_t\}$ with high probability:
\begin{align}
\label{eq:leastsqure}
\min_{h\in\mathcal{H}}\sum_{t=1}^T\|y_t - h(x_t)\|^2 \leq (1-\gamma)\sum_{t=1}^T\|y_t\|^2 + o(T).
\end{align} 
Our definition of online weak learning directly generalizes the batch weak learning definition in~\citep{grubb2011generalized} to the online setting by the additional agnostic learnability assumption as shown in Eqn.~\ref{eq:learnable}.


Note that we pick square losses (Eqn.~\ref{eq:leastsqure}) in our weak online learning definition. As we will show later, the goal is to enforce that the weak learners to accurately predict gradients, as was also originally used in the batch gradient boosting algorithm~\citep{friedman2001greedy}. Least-squares losses are also shown to be important in streaming tasks by~\citep{opauc} for their superior computational and theoretical properties. 

The above online weak learning edge definition immediately implies the following result, which is used in later proofs:
\begin{lemma}
\label{lemma:from_weak_learning}
Given the sequence of losses $\|y_t - h(x_t)\|^2$, $1\leq t\leq T$, the online weak learner generates a sequence of predictors $\{h_t\}_t$, such that:
\begin{align}
\sum_{t=1}^T 2y_t^T h_t(x_t) \geq \gamma\sum_{t=1}^T \|y_t\|^2 - R(T), \;\;\;\; \gamma\in (0,1].
\end{align} 
\end{lemma}
The above lemma can be proved by expanding the square on the LHS of Eqn.~\ref{eq:weak_learner_eq}, cancelling common terms and rearranging terms.

\subsection{Why Weak Learner Edge is Reasonable?}
We demonstrate here that the weak online learning edge assumption is reasonable. 
Let us consider the case that the hypothesis class $\mathcal{H}$ is closed under scaling (meaning if $h \in \mathcal{H}$, then for all $\alpha \in \mathbb{R}$, $\alpha h \in \mathcal{H}$) and let us assume $x\sim\ D$, and $y = f^*(x)$ for some unknown function $f^*$. We define the inner product $\langle h_1, h_2\rangle$ of any two functions $h_1,h_2$ as $\mathbb{E}_{x\sim D} [h_1(x)^T h_2(x)]$ and the squared norm $\|h\|^2$ of any function $h$ as $\langle h, h \rangle$. We assume $f^*$ is bounded in a sense $\|f^*(x)\|\leq F\in\mathbb{R}^+$. The following proposition shows that as long as $f^*$ is not perpendicular to the span of $\mathcal{H}$ ($f^*\not\perp span(\mathcal{H})$), i.e., $\exists h\in span(\mathcal{H})$ such that $\langle h, f^* \rangle \neq 0$, then we can achieve a non-zero edge:
\begin{proposition}
\label{prop:edge_example}
Consider any sequence of pairs $\{x_t, y_t\}_{t=1}^T$, where $x_t$ is i.i.d sampled from $D$, $y_t = f^*(x_t)$ and $f^*\not\perp span(\mathcal{H})$. Run any no-regret online algorithm $\mathcal{A}$ on sequence of losses $\{\|y_t - h(x_t)\|^2\}_t$ and output a sequence of predictions $\{h_t\}_t$. With probability at least $1-\delta$, there exists a weak online learning edge $\gamma\in(0,1]$, such that:
\begin{align}
&\sum_{t=1}^T \|h_t(x_t) - y_t\|^2 \leq (1-\gamma)\sum_{t=1}^T \|y_t\|^2\nonumber\\
&\;\;\;\;\;\;\; +R_{\mathcal{A}}(T)+ (2-\gamma)O\Big(\sqrt{ T\ln(1/\delta)}\Big), \nonumber
\end{align} where $R_{\mathcal{A}}(T)$ is the regret of online algorithm $\mathcal{A}$.
\end{proposition}
The proof of the above proposition can be found in Appendix. Matching to Eq.~\ref{eq:weak_learner_eq}, we have $R(T) = R_{\mathcal{A}}(T) + (2-\gamma)O\Big(\sqrt{ T\ln(1/\delta)}\Big) \in o(T)$.
In addition, the contrapositive of the proposition implies that without a positive
edge, $span(\mathcal{H})$ is orthogonal to $f^*$ so that no linear boosted ensemble can approximate $f^*$. Hence having a positive online weak learner edge is necessary for online boosted algorithms.

\section{ALGORITHM}
\subsection{Smooth Loss Functions}
We first present \algname (\algshort), an algorithm (Alg.~\ref{alg:online_gradient_boost}) that is designed for loss functions $\{\ell_t(y)\}$ that are $\lambda$-strongly convex and $\beta$-smooth. 
\begin{algorithm}[tb]
\caption{\algname (\algshort)}
 \label{alg:online_gradient_boost}
\begin{algorithmic}[1]
  \STATE {\bfseries Input:} A restricted  class $\mathcal{H}$. $N$ online weak learners $\{\mathcal{A}_i\}_{i=1}^N$. Learning rate $\eta$.
  \STATE Each weak learner initlizes a hypothesis $h_i^1\in\mathcal{H},\forall 1\leq i\leq N$.
 \FOR {t = 1 to T}
    \STATE Receive  $x_t$ and initialize $y_t^{0} = y_0$ (e.g., $y_0 = 0$). 
    \FOR {i = 1 to N}
    \label{line:GD}
        \STATE Set the partial sum $y_t^{i} = y_t^{i-1} - \eta h_i^{t}(x_t)$.
    \ENDFOR
    \STATE Predict $y_t = y_t^N$.
    \STATE $\ell_t$ is revealed and learner suffers loss $\ell_t(y_t)$.
    \FOR {i = 1 to N}
        \STATE Compute gradient w.r.t partial sum: $\nabla_i^t = \nabla\ell_t(y_t^{i-1})$.
        \label{line:gradient}
        \STATE Feed loss $\|\nabla_i^t - h_i^t(x_t)\|^2$ to $\mathcal{A}^i$.
        \label{line:feed}
        \STATE Weak learner $\mathcal{A}^i$ computes $h_i^{t+1}$ using its no-regret update procedure.
        \label{line:weak_learner_update}
    \ENDFOR
 \ENDFOR
 \STATE Set $\bar{h}_i =\frac{1}{T} \sum_{t=1}^Th_{i}^t,\forall 1\leq i\leq N$. \label{line:average}
 \STATE {\bfseries Return:$\big\{\bar{h}_1,..., \bar{h}_N$\big\}}. \label{line:stoch_return} 
\end{algorithmic}
\end{algorithm}
Alg.~\ref{alg:online_gradient_boost} is the online version of the classic batch gradient boosting algorithms~\citep{friedman2001greedy,grubb2011generalized}. Alg.~\ref{alg:online_gradient_boost} maintains $N$ weak learners. At each time step $t$, given example $x_t$, the algorithm predicts $y_t$ by linearly combining the weak learners' predictions (Line~\ref{line:GD}). Then after receiving loss $\ell_t$, for each weak learner, the algorithm computes the gradient of $\ell_t$ with respect to $y$ evaluated at the \emph{partial} sum $y_t^{i-1}$ (Line~\ref{line:gradient}) and feeds the square loss  $l_t(h)$ with the computed gradient as the regression target to weak learner $\mathcal{A}^i$ (Line~\ref{line:feed}). The weak learner $\mathcal{A}^i$ then performs its own no-regret online update to compute $h_{i}^{t+1}$ (Line~\ref{line:weak_learner_update}).

Line~\ref{line:average} and~\ref{line:stoch_return} are needed for stochastic setting. We compute the average $\bar{h}_i$ for every weak learner $\mathcal{A}_i$ in Line~\ref{line:average}. In testing time, given $x\sim D$, we predict $y$ as:
\begin{align}
\label{eq:predict_avg}
y = y_0 - \eta\sum_{i=1}^N \bar{h}_i(x). 
\end{align}

Since we penalize the weak learners by the squared deviation of its own prediction and the gradient from the previous partial sum, we essentially force weak learners to produce predictions that are close to the gradients (in a no-regret perspective). With this perspective, \algshort can be understood as using the weak learners' predictions as $N$ gradient descent steps where the gradient of each step $i$ is approximated by a weak learner's prediction (Line~\ref{line:GD}). Let us define $\Delta_0 = \sum_{t=1}^T (\ell_t(y_t^0) - \ell_t(f^*(x_t)))$, for any $f^*\in\mathcal{F}$. Namely $\Delta_0$ measures the performance of the initialization $\{y_t^0\}_t$. Under our assumption that the loss is bounded, $|\ell_t(x)|\leq B,\forall t,x$, we can simply upper bound $\Delta_0$ as $\Delta_0\leq 2BT$.  Alg.~\ref{alg:online_gradient_boost} has the following performance guarantee:
\begin{theorem}
\label{them:smooth_strongly_convex}
Assume weak learner $\mathcal{A}_i,\forall i$ has weak online learning edge $\gamma\in(0,1]$. Let $f^* = \argmin_{f\in\mathcal{F}}\sum_t\ell_t(f(x_t))$. There exists a $\eta = \frac{\gamma}{\beta(8-4\gamma)}$, for $\lambda$-strongly convex and $\beta$-smooth loss functions, $\ell _t$, such that when $T\to\infty$, Alg.~\ref{alg:online_gradient_boost} generates a sequence of predictions $\{y_t\}_t$ where:
\begin{align}
\label{eq:regret_bound_smooth}
&\frac{1}{T}[\sum_{t=1}^T\ell_t(y_t) -\sum_{t=1}^T\ell_t(f^*(x_t))]
\leq 2B(1 - \frac{\gamma^2\lambda}{16\beta})^N. 
\end{align} For stochastic setting where $(x_t, \ell_t)\sim D$ independently, we have when $T\to\infty$:
\begin{align}
\label{eq:smooth_result_stoch}
\mathbb{E}\big[\ell\big(y_0-\eta\sum_{i=1}^N \bar{h}_i(x)\big) -  \ell(f^*(x)) \big]\leq 2B(1-\frac{\gamma^2\lambda}{16\beta})^N.
\end{align}
\end{theorem} 
The expectation in Eqn.~\ref{eq:smooth_result_stoch} of the above theorem is taken over the randomness of the sequence of pairs of loss and samples $\{\ell_t, x_t\}_{t=1}^T$ (note that $\bar{h}_i$ is dependent on $\ell_1,x_1,...,\ell_{T},x_{T}$) and $\ell, x$. 
Theorem~\ref{them:smooth_strongly_convex} shows that with infinite amount samples the average regret decreases exponentially as we increase the number of weak learners. 
This performance guarantee is very similar to classic batch boosting algorithms \citep{schapire2012boosting,grubb2011generalized}, where the empirical risk decreases exponentially with the number of algorithm iterations, i.e., the number of weak learners. 
Theorem~\ref{them:smooth_strongly_convex} mirrors that of Theorem 1 in \citep{beygelzimer2015online}, which bounds the regret of the Frank-Wolfe-based Online Gradient Boosting algorithm. Our results utilize the additional assumptions that the losses $\ell_t$ are strongly convex and that the weak learners have edge, allowing us to shrink the average regret exponentially with respect to N, while the average regret in \citep{beygelzimer2015online} shrinks in the order of $1/N$ (though this dependency on $N$ is optimal under their setting).

Proof of Theorem~\ref{them:smooth_strongly_convex}, detailed in Appendix~\ref{sec:proof_smooth_strongly_convex}, weaves our additional assumptions 
into the proof framework of gradient descent on smooth losses. In particular, using weak learner edge assumption, we derive  
Lemma~\ref{lemma:from_weak_learning} and the Lemma~\ref{lemma:helper_1} to relate parts of the strong smoothness expansion of the losses to the norm-squared of the gradients $\Vert \nabla \ell _t(y_t^i) \Vert^2$, which is an upper bound of \mbox{$2\lambda (\ell_t(y_t^i) - \ell_t(f^*(x_t)))$} due to strong convexity. Using this observation,
we can relate the total regret of the ensemble of the first $i$ learners, 
\mbox{$\Delta_i = \sum_{t=1}^T (\ell_t(y_t^i) - \ell_t(f^*(x_t)))$}, with the regret from using $i+1$ learners, $\Delta_{i+1}$, and show that $\Delta_{i+1}$ shrinks $\Delta_i$ by a constant fraction while only adding a small term $O(R(T)) \in o(T)$. Solving the recursion on the sequence of $\Delta_{i}$, we arrive at the final exponentially decaying regret bound in the number of learners. 

\paragraph{Remark} Due to the weak online learning edge assumption, the regret bound shown in Eqn.~\ref{eq:regret_bound_smooth} and the risk bound shown in Eqn.~\ref{eq:smooth_result_stoch} are stronger than typical bounds in classic online learning, in a sense that we are competing against $f^*$ that could potentially be much more powerful than any hypothesis from $\mathcal{H}$. For instance when the loss function is square loss $\ell(f(x)) = \|f(x) - z\|^2$, Theorem~\ref{them:smooth_strongly_convex} essentially shows that the risk of the boosted hypothesis $\mathbb{E}[\|y_0-\eta\sum_{i=1}^N\bar{h}_i(x) - z\|^2]$ approaches to zero as $N$ approaches to infinity, under the assumption that $\mathcal{A}_i,\forall i$ have no-zero weak learning edge (e.g.,$f^* \in span(\mathcal{H})$). Note that this is analogous to the results of classification based batch boosting \citep{freund1995desicion,grubb2011generalized} and online boosting \citep{beygelzimer2015optimal}: as number of weak learners increase, the average number of prediction mistakes approaches to zero. In other words, with the corresponding edge assumptions, these batch/online boosting classification algorithms can compete against any arbitrarily powerful classifier that always makes zero mistakes on any given training data.

\subsection{Non-smooth Loss Functions}
\label{subsec:non-smooth}
The regret bound shown in Theorem~\ref{them:smooth_strongly_convex} only applies for strongly convex and smooth loss functions. In fact, one can show that Alg.~\ref{alg:online_gradient_boost} will fail for general non-smooth loss functions. We can construct a sequence of non-smooth loss functions and a special weak hypothesis class $\mathcal{H}$, which together show that the regret of Alg.~\ref{alg:online_gradient_boost} grows linearly in the number of samples, regardless of the number of weak learners. We refer readers to Appendix~\ref{sec:counter_example} for more details. 

\begin{algorithm}[tb]
\caption{\algname (\algshort) for non-smooth loss (Residual Projection)}
 \label{alg:online_gradient_boost_non_smooth_residual}
\begin{algorithmic}[1]
  \STATE {\bfseries Input:} A restricted  class $\mathcal{H}$. $N$ online weak learners $\{\mathcal{A}_i\}_{i=1}^N$. Learning rate schedule $\{\eta_i\}_{i=1}^N$.
  \STATE  $\forall i,\mathcal{A}_i$ initializes a hypothesis $h_{i}^1 \in \mathcal{H}$.
  \FOR {t = 1 to T}
    \STATE Receive $x_t$ and initialize $y_t^{0} = y_0$ (e.g., $y_{0} = 0$). 
    \FOR {i = 1 to N}
    \label{line:GD}
        \STATE Set the projected partial sum $y_t^{i} = \Pi_{\mathcal{Y}}(y_t^{i-1} - \eta_i h_i^t(x_t)$).
    \ENDFOR
    \STATE Predict ${y}_t = \frac{1}{N}\sum_{i=0}^N {y}_t^i$
    \STATE The loss $\ell_t$ is revealed and compute loss $\ell_t(y_t)$.
    \STATE Set initial residual $\Delta_0^t = 0$. 
    \FOR {i = 1 to N}
        \STATE Compute subgradient w.r.t. partial sum: 
            $\nabla_i^t = \nabla\ell_t(y_t^{i-1})$. 
        \STATE Feed loss $\big\|(\Delta^t_{i-1} + \nabla^t_i) - h(x)\big\|^2$ to $\mathcal{A}^i$.
        \label{line:feed_residual_loss}
        \STATE Update residual: $\Delta_i^t = \Delta^t_{i-1} + \nabla^t_i - h_i^t(x_t)$.
        \label{line:residual_update}
        \STATE Weak learner $\mathcal{A}^i$ computes $h_i^{t+1}$ using its no-regret update procedure.
    \ENDFOR
 \ENDFOR
 \STATE {\bfseries Return: $h_t^i, 1\leq i\leq N, 1\leq t\leq T$.\label{line:stoch_return_non_smooth}}
\end{algorithmic}
\end{algorithm}

\begin{algorithm}[tb]
\caption{\algshort (Residual Projection) for testing}
 \label{alg:online_gradient_boost_non_smooth_residual_test}
\begin{algorithmic}[1]
  \STATE {\bfseries Input:} Test sample $x$ and $h_t^i, 1\leq i\leq N, 1\leq t\leq T$ from the output of Alg.~\ref{alg:online_gradient_boost_non_smooth_residual}.
  \FOR {t = 1 to T}
    \FOR {i = 1 to N} \label{line:for_loop}
        \STATE $y_t^{i} = \Pi_{\mathcal{Y}}(y_t^{i-1} - \eta_i h_i^t(x)$).
    \ENDFOR
    \STATE $y_t = \frac{1}{N}\sum_{i=0}^N y_t^i$. 
  \ENDFOR
  \STATE {\bfseries Predict: $y = \mathcal{T}(x) = \frac{1}{T}\sum_{t=1}^T y_t$}.  \label{line:stoch_return_non_smooth}
\end{algorithmic}
\end{algorithm}

Our next algorithm, Alg.~\ref{alg:online_gradient_boost_non_smooth_residual}, extends \algshort (Alg.~\ref{alg:online_gradient_boost}) to handle strongly convex but non-smooth losses. Instead of training each weak learner to fit the subgradients of non-smooth loss with respect to current prediction, we instead keep track of a residual $\Delta_i$\footnote{Note the abusive notation. For the non-smooth loss setting (Alg.~\ref{alg:online_gradient_boost_non_smooth_residual}), $\Delta_i$ does not refer to the regret of the ensemble's regret with the $i$-th as used in the analysis of Alg.~\ref{alg:online_gradient_boost}} that accumulates the difference between the subgradients, $\nabla_k$, and the fitted prediction $h_k(x_t)$, from $k=1$ up to $i-1$. Instead of fitting the predictor $h_{i+1}$ to match the subgradient $\nabla_{i+1}$, we fit it to match the sum of the subgradient and the residuals, $\nabla_{i+1} + \Delta_{i}$. More specifically, in Line~\ref{line:feed_residual_loss} of Alg.~\ref{alg:online_gradient_boost_non_smooth_residual}, for each weak learner $\mathcal{A}^i$, we feed a square loss with the sum of residual and the gradient as the regression target. Then Line~\ref{line:residual_update} sets the new the residual $\Delta_i^t$ as the difference between the target $(\Delta_{i-1}^t + \nabla_i^t)$ and the weak learner $\mathcal{A}^i$'s prediction $h_i^t(x_t)$.

The last line of Alg.~\ref{alg:online_gradient_boost_non_smooth_residual} is needed for stochastic setting where $(\ell_t, x_t)\sim D$ i.i.d. In test, given sample $x\sim D$, we predict $y$ using $h_t^i, \forall i,t$ in procedure shown in Alg.~\ref{alg:online_gradient_boost_non_smooth_residual_test}. For notation simplicity, we denote the testing procedure shown in Alg.~\ref{alg:online_gradient_boost_non_smooth_residual_test} as $\mathcal{T}(x)$, which $\mathcal{T}$ explicitly depends on the returns $h_t^i, 1\leq i\leq N, 1\leq t\leq T$ from \algshort (Residual Projection). Since it's impractical to store and apply all $TN$ models, we follow a common stochastic learning technique which uses the final predictor at time $T$ for testing (e.g., \cite{SVRG}) in the experiment section (i.e., simply set $t = T$ in Line~\ref{line:for_loop} in Alg.~\ref{alg:online_gradient_boost_non_smooth_residual_test}). 
In practice, if the learners converge and $T$ is large, the average and final predictions are close.

Intuitively, this approach prevents the weak learners from consistently failing to match a certain direction of the subgradient as the net error in the direction is stored in residual. By the assumption of weak learner edge, the directions will be approximated. We also note that if we assume the subgradients are bounded, then the residual magnitudes increase at most linearly in the number of weak learners. Simultaneously, each weak learner shrinks the residual by at least a constant factor due to the assumption of edge. Hence, we expect the residual to shrink exponentially in the number of learners. Utilizing this observation, we arrive at the following performance guarantee:
\begin{theorem}
\label{them:non_smooth_strongly_convex}
Assume the loss
$\ell_t$ is $\lambda$-strongly convex for all $t$ with bounded gradients, $\Vert \nabla \ell_t(y) \Vert \leq G$ for all $y$, and each weak learner $\mathcal{A}_i$ has edge $\gamma\in(0,1]$. Let $\mathcal{F}$ be a function space, and $\mathcal{H} \subset \mathcal{F}$ be a restriction of $\mathcal{F}$ 
Let 
$f^* = \argmin _{f \in \mathcal{F}} \frac{1}{T} \sum _{t=1}^T \ell_t(f(x_t))$ be the 
optimal predictor in $\mathcal{F}$ in hindsight. Let $c = \frac{2}{\gamma} -1$. Let
step size be $\eta_i = \frac{1}{\lambda i}$. When $T\to \infty$, we have:
\begin{align}
&\frac{1}{T} \sum_{t=1}^T \left( \ell_t(y_t) - \ell_t(f^*(x_t)) \right) \leq \frac{4c^2G^2}{\lambda N} (1 + \ln N + \frac{1}{8N}). \label{eq:nonsmooth_result}
\end{align}
For stochastic setting where $(x_t, \ell_t)\sim D$ independently, when $T\to\infty$ we have:
\begin{align}
\label{eq:non_smooth_result_stoch}
\mathbb{E}\big[\ell(\mathcal{T}(x)) -  \ell(f^*(x)) \big]\leq \frac{4c^2G^2}{\lambda N} (1 + \ln N + \frac{1}{8N}).  \nonumber
\end{align}
\end{theorem}
The above theorem shows that the average regret of Alg.~\ref{alg:online_gradient_boost_non_smooth_residual} is $O(\ln N/N)$ with respect to the number $N$ of weak learners, which matches the regret bounds of Online Gradient Descent for strongly convex loss. 
The key idea for proving Theorem~\ref{them:non_smooth_strongly_convex} is to combine our online weak learning edge definition with the proof framework of Online Gradient Descent for strongly convex loss functions from~\citep{hazan2007logarithmic}. The detailed proof can be found in Appendix~\ref{sec:proof_non_smooth_strongly_convex}.

\section{EXPERIMENTS}
\begin{figure*}[t]
    \centering
    \begin{subfigure}[l]{0.42\textwidth}
        \centering
        \includegraphics[width=0.49\textwidth,keepaspectratio]{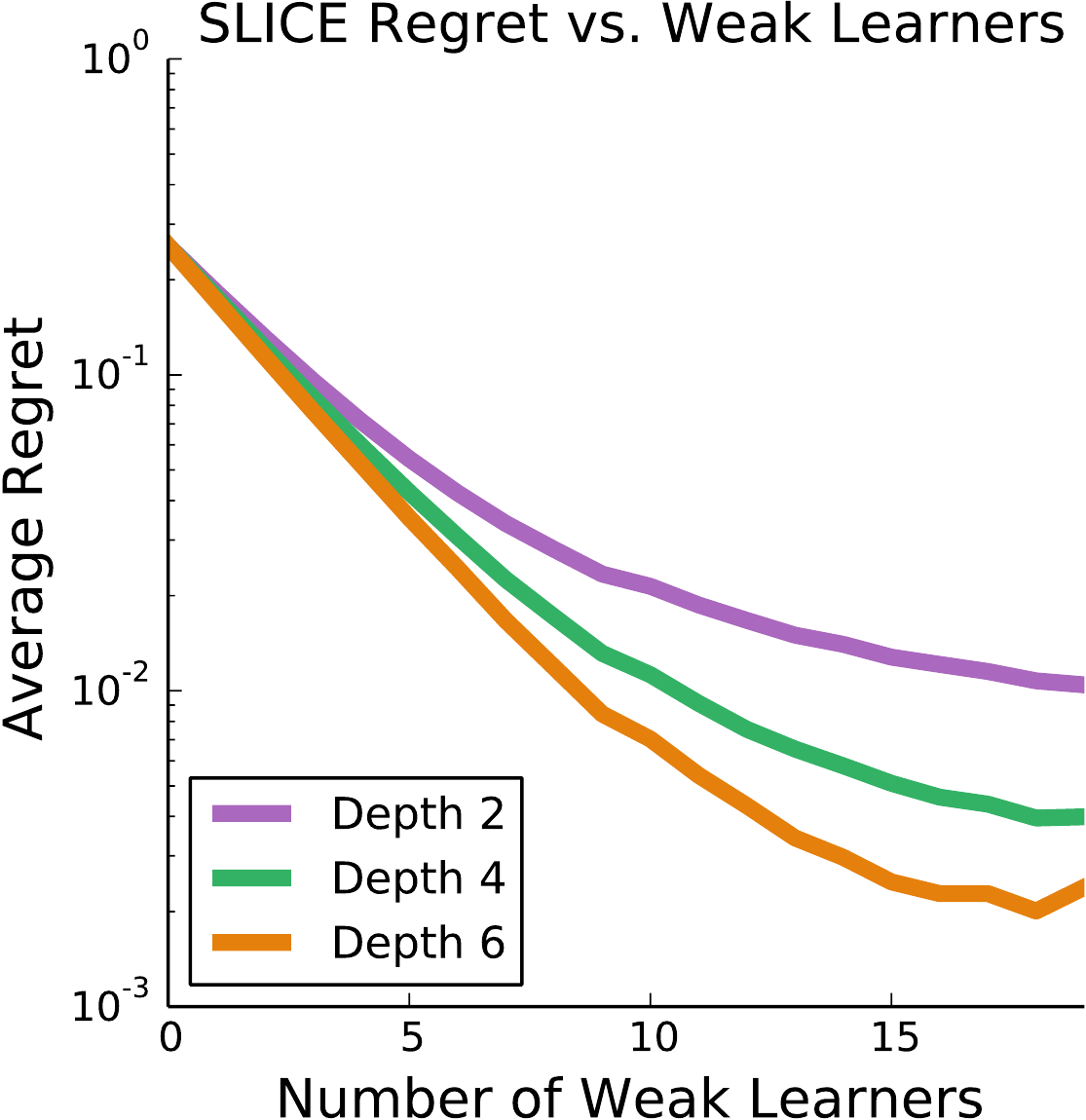}
        \includegraphics[width=0.49\textwidth,keepaspectratio]{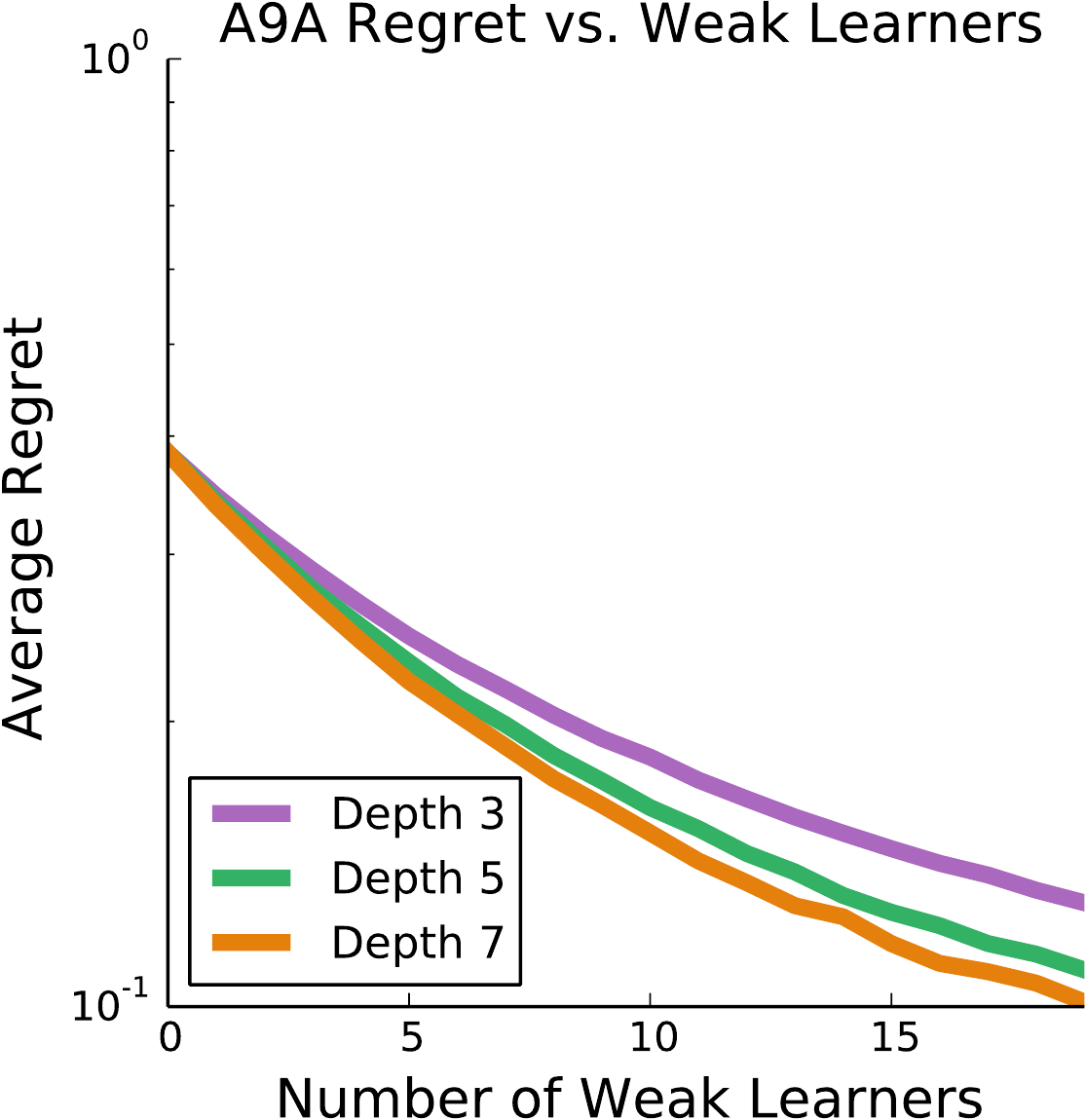}
        \caption{Regret versus number of weak learners}
        \label{fig:regret_learners}
    \end{subfigure}
    \begin{subfigure}[l]{0.42\textwidth}
        \centering
        \includegraphics[width=0.49\textwidth,keepaspectratio]{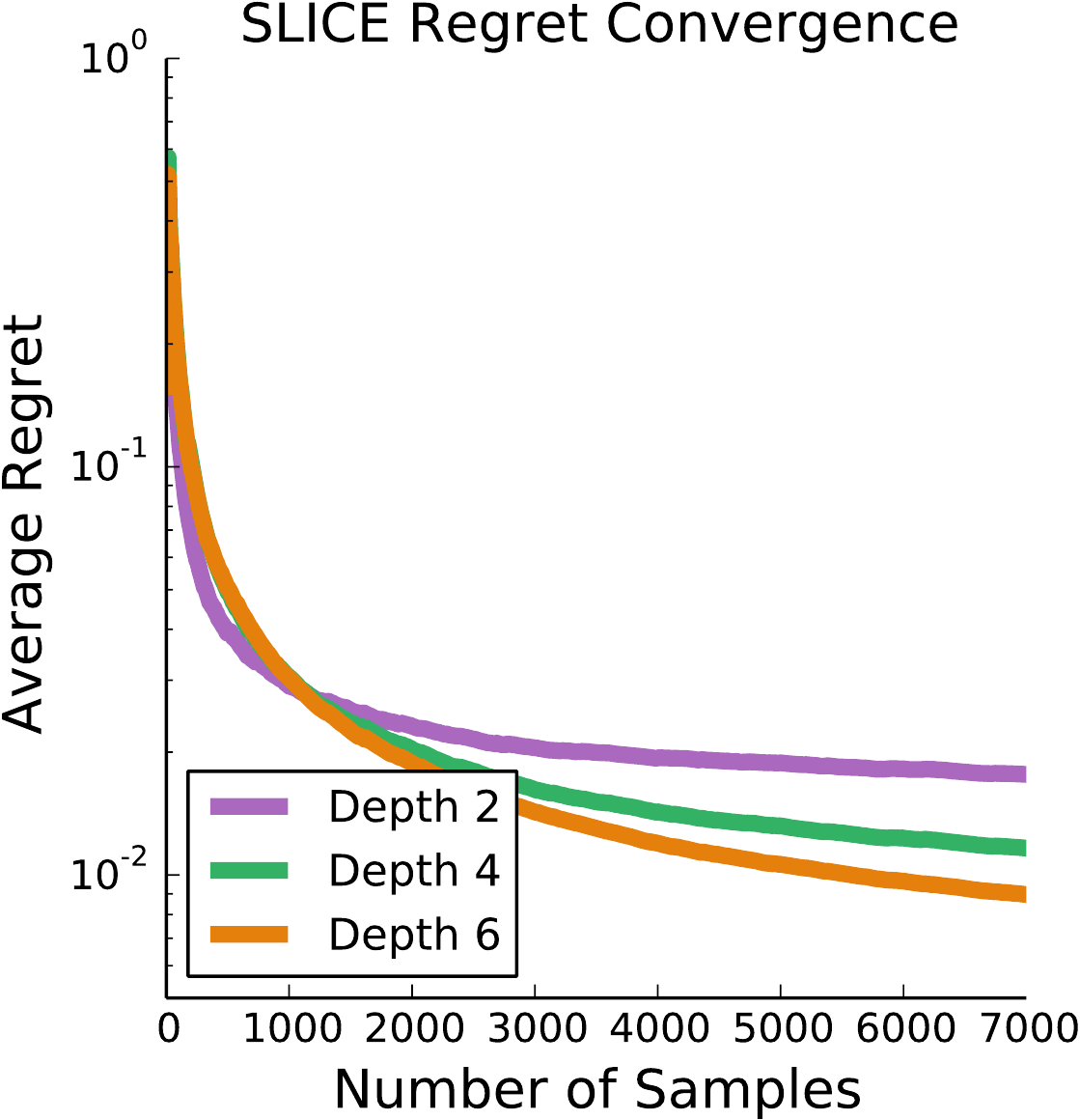}
        \includegraphics[width=0.49\textwidth,keepaspectratio]{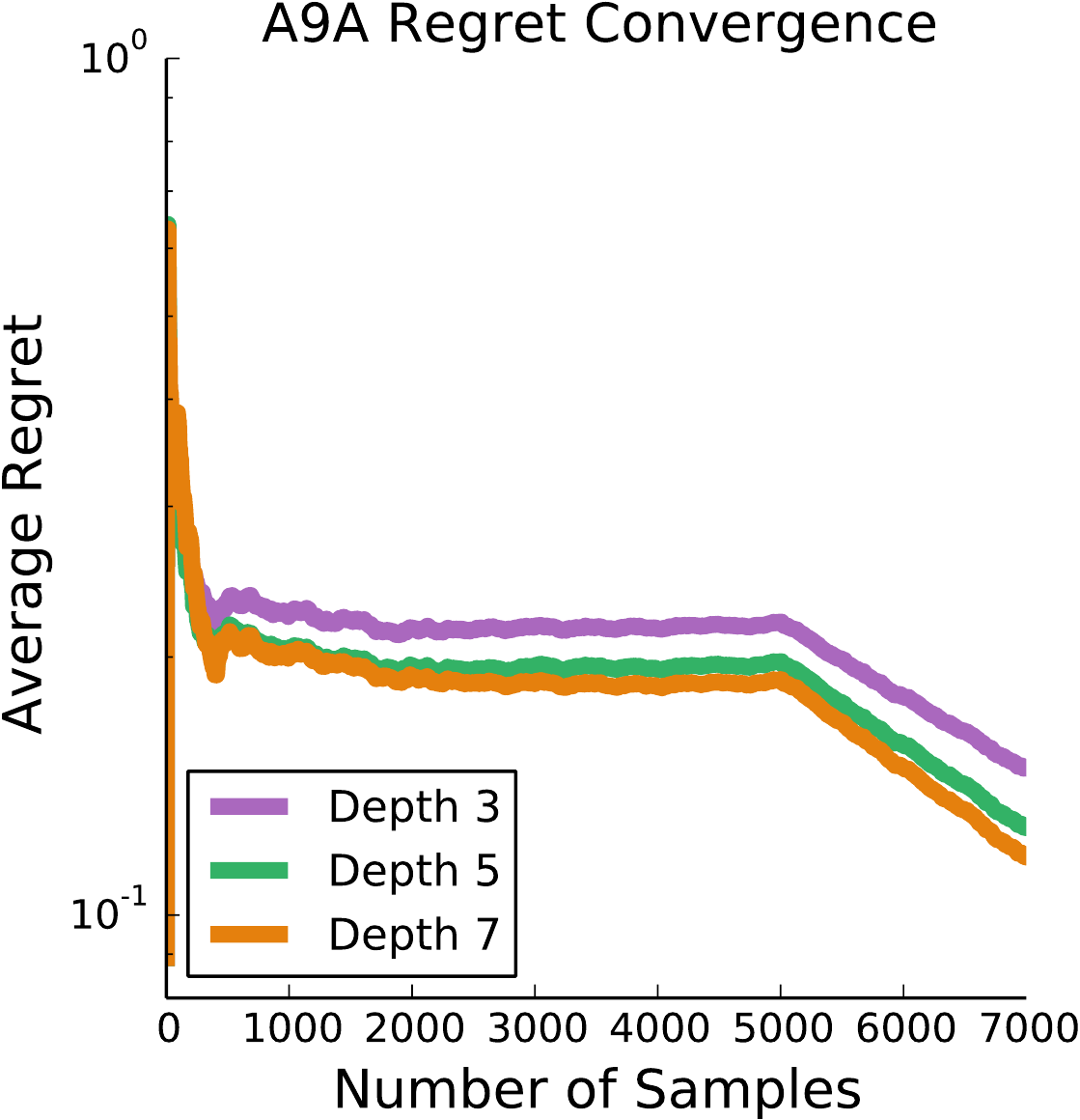}
        \caption{Regret versus number of samples seen}
        \label{fig:regret_samps}
    \end{subfigure}
    \vspace{-2pt}
    \caption{Average regret of \algshort with regression trees with various depths on SLICE and A9A datasets.}
    \label{fig:regret}
    \vspace{-5pt}
\end{figure*}

We demonstrate the performance of our \algname using the following UCI datasets~\citep{Lichman:2013}: YEAR, ABALONE, SLICE, and A9A~\citep{adult} as well as the  MNIST~\citep{mnist} dataset. If available, we use the given train-test split of each data-set. Otherwise, we create a random 90\%-10\% train-test split. 

    
    
    

\subsection{Experimental Analysis of Regret Bounds}
We first demonstrate the relationships between the regret bounds shown in Eqn.~\ref{eq:regret_bound_smooth} and the parameters including the number of weak learners, the number of samples and edge $\gamma$. We compute the regret of \algshort with respect to a deep regression tree (depth$\geq$ 15), which plays the $f^*$ in Eqn.~\ref{eq:regret_bound_smooth}. 
We use regression trees as the weak learners. We assume that deeper trees have higher edges $\gamma$ because they empirically fit training data better. We show how the regret relates to the trees' depth, the number of weak learners $N$ (Fig.~\ref{fig:regret_learners}) and the number of samples $T$ (Fig.~\ref{fig:regret_samps}).

For the experimental results shown in Fig.~\ref{fig:regret}, we used smooth loss functions with $L_2$ regularization (see Appendix~\ref{sec:implementation} for more details). We use logistic loss and square loss for binary classification (A9A) and regression task (SLICE), respectively. For each regression tree weak learner, Follow The Regularized Leader (FTRL)~\citep{shalev2011online} was used as the no-regret online update algorithm with regularization posed as the depth of the tree. Fig.~\ref{fig:regret_learners} shows the relationship between the number of weak learners and the average regret given a fixed total number of samples. The average regret decreases as we increase the number of weak learners. We note that the curves are close to linear at the beginning, matching our theoretical analysis that the average regret decays exponentially (note the y-axis is log scale) with respect to the number of weak learners. This shows that \algshort can significantly boost the performance of a single weak learner.

To investigate the effect of the edge parameter $\gamma$, we additionally compute the average regret in Fig.~\ref{fig:regret} as the depth of the regression tree is increased. The tree depth increases the model complexity of the base learner and should relate to a larger $\gamma$ edge parameter. From this experiment, we see that the average regret shrinks as the depth of the trees increases.

Finally, Fig.~\ref{fig:regret_samps} shows the convergence of the average regret with respect to the number of samples. We see that more powerful weak learners (deeper regression trees) results in faster convergence of our algorithm. We ran Alg.~\ref{alg:online_gradient_boost_non_smooth_residual} on A9A with hinge loss and SLICE with $L1$ (least absolute deviation) loss and observed very similar results as shown in Fig.~\ref{fig:regret}. 

\subsection{Batch Boosting vs. Streaming Boosting}

\begin{figure*}[t!]
    \centering
    \begin{subfigure}[l]{0.28\textwidth}
        \centering
        \includegraphics[trim={1cm 0.6cm 1cm 1cm},clip,keepaspectratio,height=2.75cm]{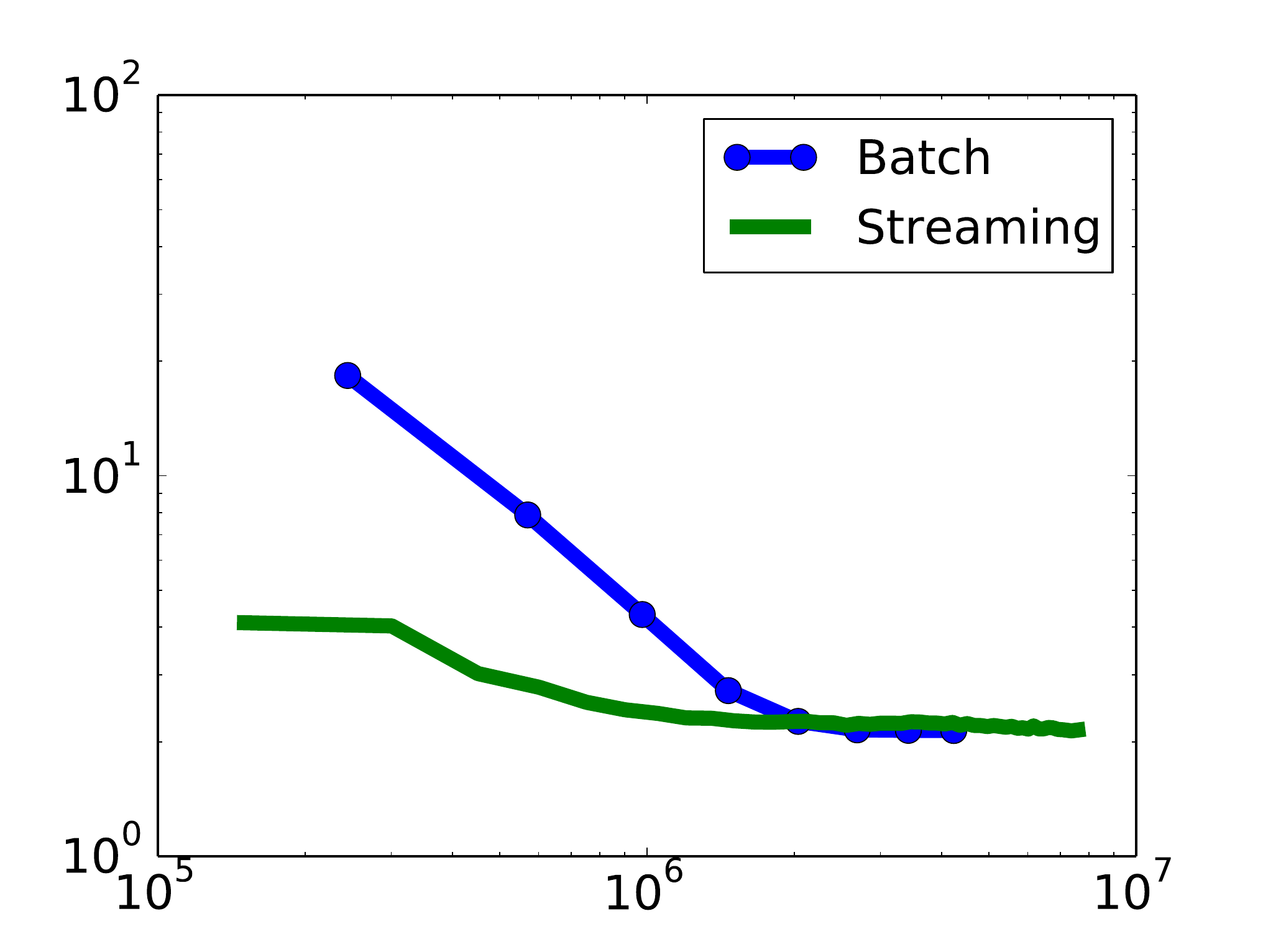}
        \caption{ABALONE N=8}
        \label{fig:abalone}
    \end{subfigure}
    \begin{subfigure}[l]{0.28\textwidth}
        \centering
        \includegraphics[trim={1cm 0.6cm 1cm 1cm},clip,keepaspectratio,height=2.75cm]{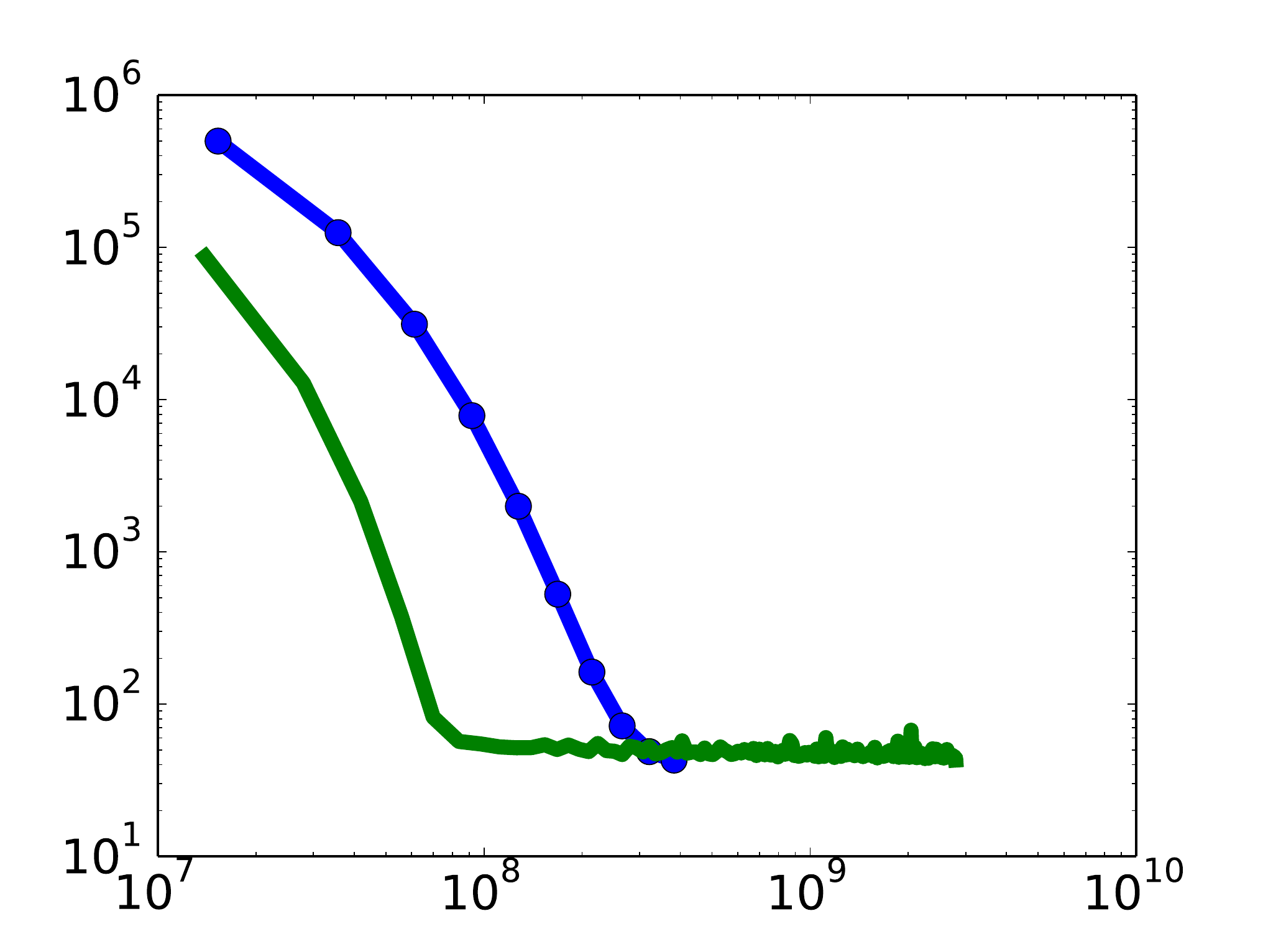}
        \caption{YEAR N=10}
        \label{fig:year}
    \end{subfigure}
    \begin{subfigure}[l]{0.28\textwidth}
        \centering
        \includegraphics[trim={1cm 0.6cm 1cm 1cm},clip,keepaspectratio,height=2.75cm]{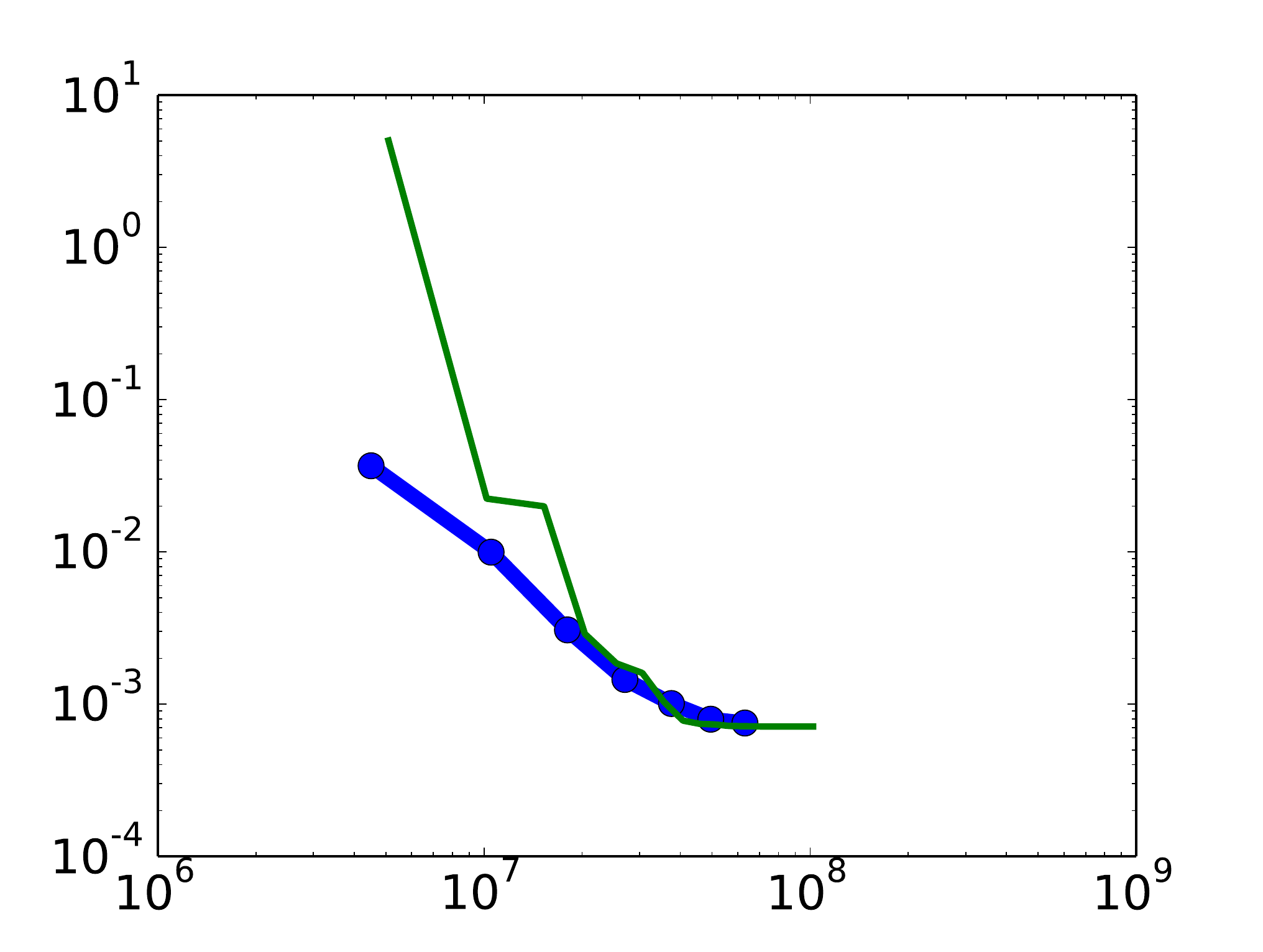}
        \caption{SLICE N=8}
        \label{fig:slice}
    \end{subfigure}
    
    \begin{subfigure}[l]{0.28\textwidth}
        \centering
        \includegraphics[trim={1cm 0.6cm 1cm 1cm},clip,keepaspectratio,height=2.75cm]{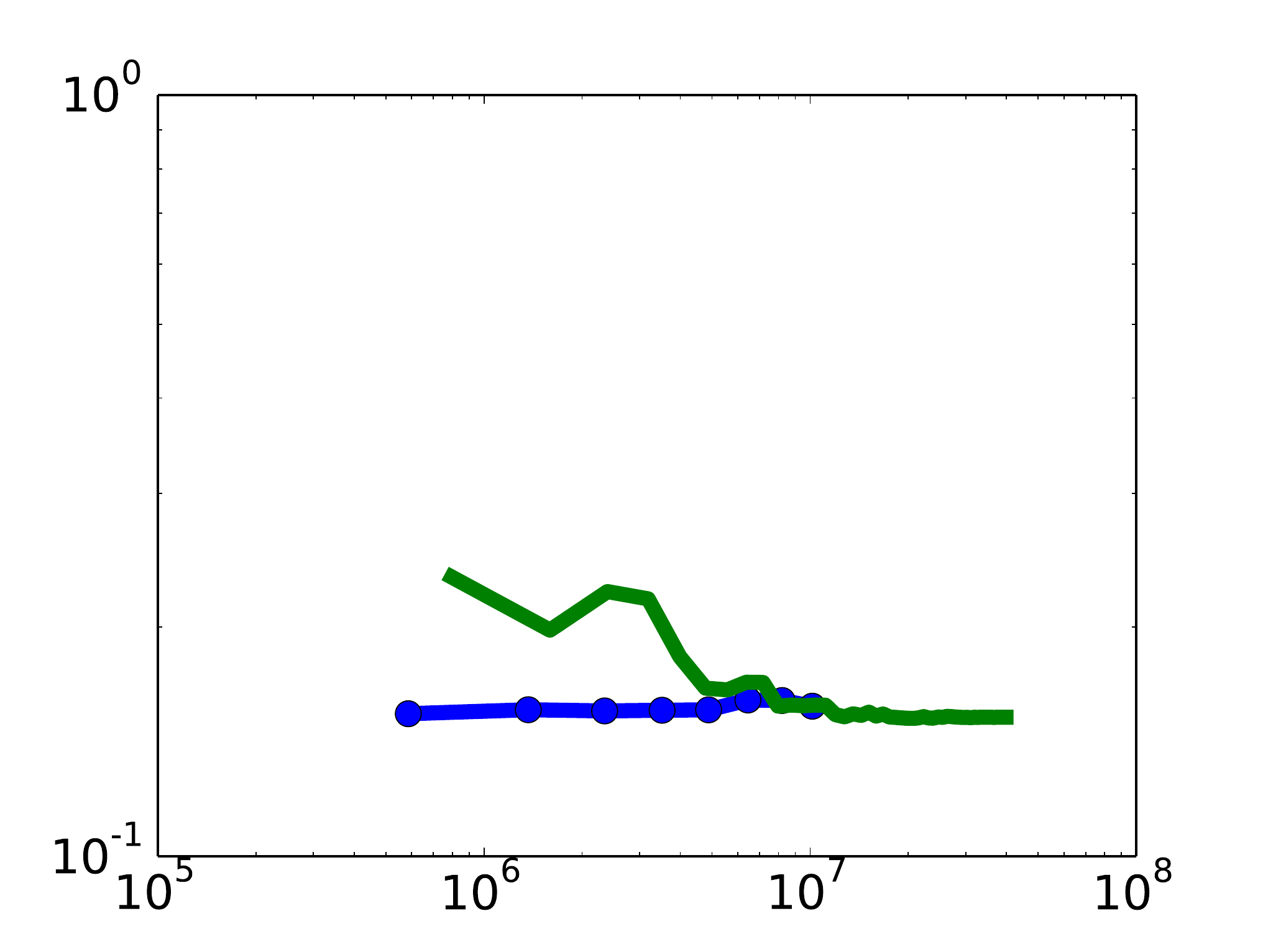}
        \caption{A9A N=4}
        \label{fig:a9a}
    \end{subfigure}
    \begin{subfigure}[l]{0.29\textwidth}
        \centering
        \includegraphics[trim={1cm 0.6cm 1cm 1cm},clip,keepaspectratio,height=2.75cm]{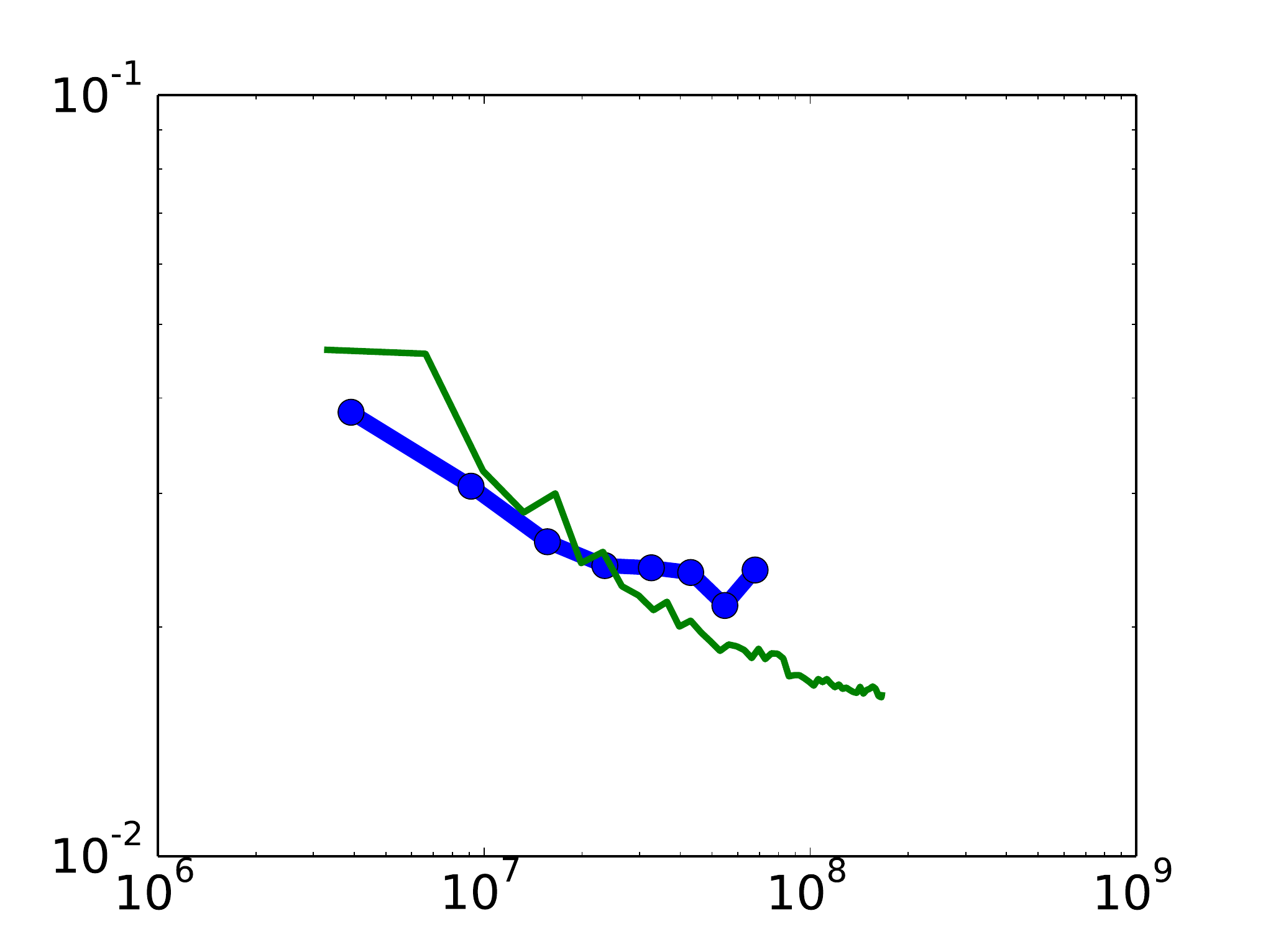}
        \caption{MNIST N=10}
        \label{fig:mnist}
    \end{subfigure}
    \begin{subfigure}[l]{0.27\textwidth}
        \centering
        \includegraphics[trim={1cm 0.6cm 1cm 1cm},clip,keepaspectratio,height=2.75cm]{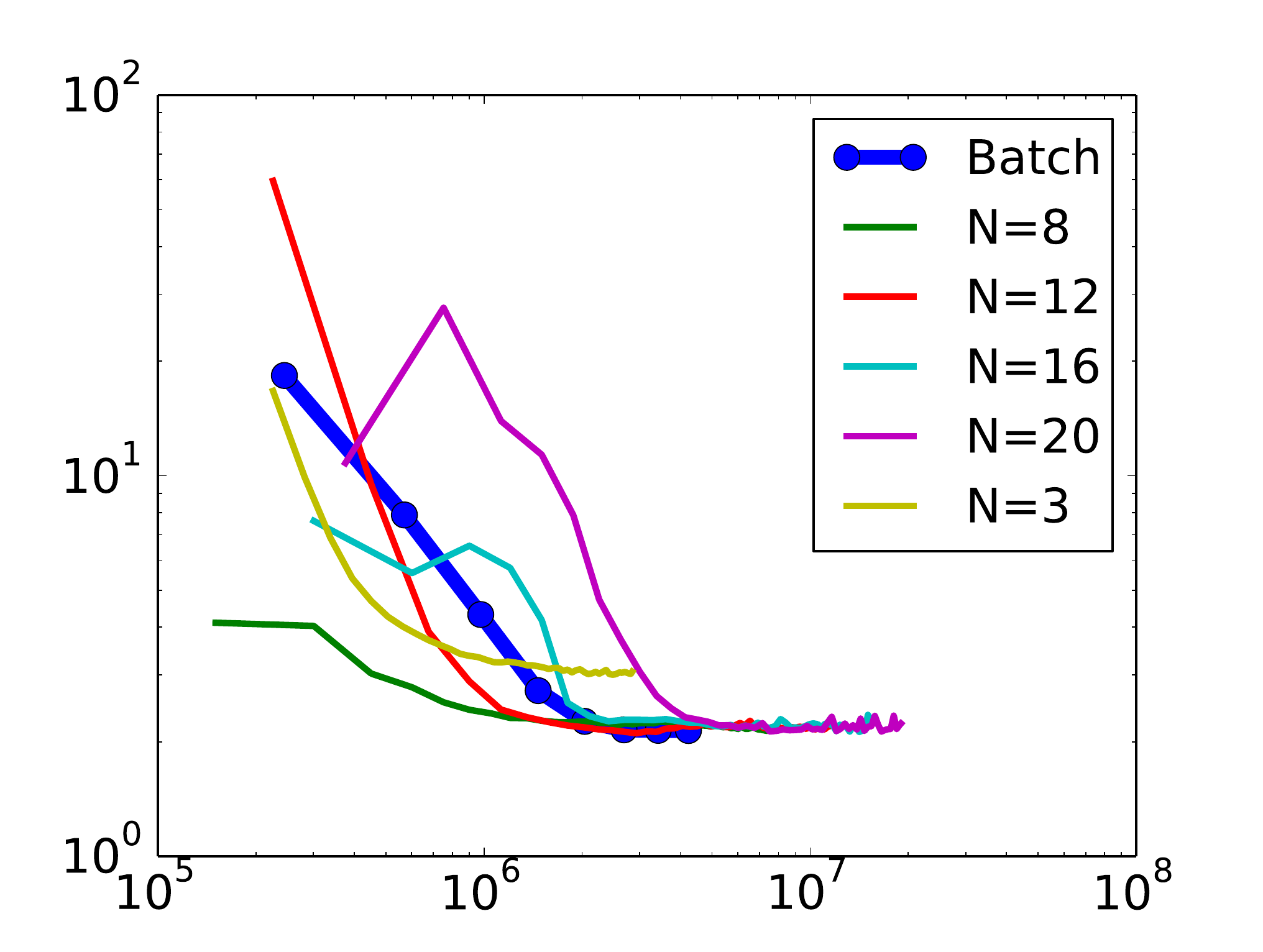}
        \caption{ABALONE with varied N}
        \label{fig:abalone_vary_n}
    \end{subfigure}
    \vspace{-5pt}
    \caption{Log-log plots of test-time loss vs. computation complexity on various data-sets. The x-axis represents computation complexity measured by number of weak leaner predictions; the y-axis measures square loss for regression tasks (ABALONE, SLICE and YEAR), and classification error for A9A and MNIST. }
    \label{fig:convg}
    \vspace{-5pt}
\end{figure*}


We next compare batch boosting to \algshort using two-layer neural networks as weak learners\footnote{The number of hidden units by data-set: ABALONE, A9A: 1; YEAR, SLICE: 10; MNIST: 5x5 convolution with stride of 2 and 5 output channels. Sigmoid is used as the activation for all except SLICE, which uses leaky ReLU. } and see that \algshort reaches  similar final performance as the batch boosting algorithm albeit with less training computation. 
As stated in Sec~\ref{subsec:non-smooth}, we report $h_T^i$ instead of $\bar{h}_i$ for \algshort, since at convergence the average prediction is close to the final prediction, and the latter is impractical to compute.  
We implement our baseline, the classic batch gradient boosting (\textbf{GB}) ~\citep{friedman2001greedy}, by optimizing each weak learner until convergence in order.
In both GB and \algshort, we train weak learners using ADAM~\citep{kingma2014adam} optimization and use the default random parameter initialization for NN.

We analyze the complexity of training \algshort and GB. We define the prediction complexity of one weak learner as the \textit{unit 
cost}, since the training run-time complexity almost equates the total complexity of weak learner predictions and updates. Our choice of weak learner and update method (two-layer networks and ADAM) determines that updating a weak learner is about two units cost. 
In training using \algshort, each of the $T$ data samples triggers predictions and updates with all $N$ of the weak learners. This results in a training computational complexity of $3TN = O(TN)$. For GB, let $T_B$ be the samples needed for each weak learner to converge. Then the complexity of training GB is $T_B \sum _{i=1}^N i +  2T_B N \simeq \frac{1}{2}T_BN^2= O(T_B N^2)$, because when training weak learner $i$, all previous $i-1$ weak learners must also predict for each data point\footnote{Saving previous predictions is disallowed, because data may not be revisited in an actual streaming setting.}. 
Hence, \algshort and GB will have the same training complexity if $T_B \simeq \frac{6T}{N} = \Theta(\frac{T}{N})$. In our experiments we observe weak learners typically converge less than $\frac{T}{N}$ samples, but our following experiment shows that \algshort still can converge faster overall.  



Fig.~\ref{fig:convg} plots the test-time loss versus training computation, measured by the unit cost.
Blue dots highlights when the weak learners are added in GB. We first note that \algshort successfully converges to the results of GB in all cases, supporting that \algshort is a truly a streaming conversion of GB. 
As it takes many weak learners to achieve good performance on ABALONE and YEAR, we observe that \algshort converges with less computation than GB. On A9A, however, GB is more computationally efficient than SGB, because the first weak learner in GB already performs well and learning a single weak learner for GB is faster than simultaneously optimizing all $N=8$ weak learners with \algshort.  
This suggests that if we initially set $N$ too big, \algshort could be less computationally efficient. In fact Fig.~\ref{fig:abalone_vary_n} shows that very larger $N$ causes slower convergence to the same final error plateau. On the other hand, small $N$ ($N=3$) results in worse performance. 
We specify the chosen $N$ for \algshort in Fig.~\ref{fig:convg}, and they are around the number of weak learners that GB requires to converge and achieve good performance. 
We also note that \algshort has slower initial progress compared to GB on SLICE in Fig.~\ref{fig:slice} and MNIST in Fig.~\ref{fig:mnist}. This is an understandable result as \algshort has a much larger pool of parameters to optimize.
Despite this initial disadvantage, \algshort surpasses GB and converges faster overall, suggesting the advantage of updating all the weak learners together.
In practice, if we do not have a good guess of $N$, we can still use SGB to add multiple weak learners at a time in GB to speed up convergence. 
Table~\ref{table:validation} records the test error (square error for regression and error ratio for classification) of the neural network base learner, GB, and \algshort. We observe that \algshort achieves test errors that are competitive with GB in all cases. 

\begin{table}[h!]
	\centering
    \ra{1.3}
    \resizebox{0.45\textwidth}{!}{
    \begin{tabular}{@{} l l l l l  @{} }
    \toprule
    & \textbf{Base} & \textbf{GB} & \textbf{\algshort}\\
    \midrule
    ABALONE (regression) & 8.2848 & 2.1411  &  2.1532\\
    YEAR (regression) & $4.99 \times 10^5$ &42.8976 & 43.0573 \\
    SLICE (regression) & 0.036045 & 0.000755 & 0.000713\\
    A9A (classification)  & 0.1547 & 0.1579 & 0.1523&\\ 
    MNIST (classification) & 0.163280 & 0.019320 & 0.016320\\  
    \bottomrule
    \end{tabular}}
    \vspace{0.5em}
    \caption{Average test-time loss: square error for regression, and error rate for classification.}
    \label{table:validation}
\end{table}

\section{CONCLUSION}
In this paper, we present \algshort for online convex programming. By introducing an online weak learning edge definition that naturally extends the edge definition from batch boosting to the online setting and by using square loss, we are able to boost the predictions from weak learners in a gradient descent fashion. Our \algshort algorithm guarantees exponential regret shrinkage in the number $N$ of weak learners for strongly convex and smooth loss functions. We additionally extend \algshort for optimizing non-smooth loss function, which achieves $O(\ln N/N)$ no-regret rate. Finally, experimental results support the theoretical analysis.

Though our \algshort algorithm currently utilizes the procedure of gradient descent to combine the weak learners’ predictions, our online weak learning definition and the design of square loss for weak learners leave open the possibility to leverage other gradient-based update procedures such as accelerated gradient descent, mirror descent, and adaptive gradient descent for combining the weak learners' predictions. 

\section*{Acknowledgements}
\vspace{-5pt}
This material is based upon work supported in part by: Echo's Grant name, DARPA ALIAS contract number HR0011-15-C-0027, and National Science Foundation Graduate Research Fellowship Grant No. DGE-1252522.

\medskip
{\small
\bibliographystyle{abbrvnat}
\bibliography{AISTATS}}

\newpage
\onecolumn
\appendix
\textbf{Supplementary Material for Gradient Boosting on Stochastic Data Streams}

\section{Proof of Proposition~\ref{prop:edge_example}}
\begin{proof}
Given that a no-regret online learning algorithm $\mathcal{A}$ running on sequence of loss $\|h(x_t) - y_t\|^2$, we have can easily see that Eqn.~\ref{eq:learnable} holds as:
\begin{align}
\sum_{t=1}^T\|h_t(x_t) - y_t\|^2\leq \min_{h\in\mathcal{H}}\sum_{t=1}^T \|h(x_t) - y_t\|^2 + R_{\mathcal{A}}(T), 
\end{align}where $R_{\mathcal{A}}(T)$ is the regret of $\mathcal{A}$ and is $o(T)$. To prove Proposition~\ref{prop:edge_example}, we only need to show that Eqn.~\ref{eq:leastsqure} holds for some $\gamma\in(0,1]$. This is equivalent to showing that there exist a hypothesis $\tilde{h}\in\mathcal{H}$ ($\|\tilde{h}\| = 1$), such that $\langle \tilde{h}, f^* \rangle > 0$. To see this equivalence, let us assume that $\langle \tilde{h}, f^*/\|f^*\| \rangle = \epsilon > 0$. Let us set $h^* = \epsilon\|f^*\| \tilde{h}$. Using Pythagorean theorem, we can see that $\| h^* - f^* \|^2 = (1-\epsilon^2)\|f^*\|^2$. Hence we get $\gamma$ is at least $\epsilon^2$, which is in $(0,1]$.

Now since we assume that $f^*\not\perp span(\mathcal{H})$, then there must exist ${h}'\in \mathcal{H}$, such that $\langle f^*, h'\rangle\neq 0$, otherwise $f^*\perp\mathcal{H}$. Consider the hypothesis $h'/\|h'\|$ and $-h'/\|h'\|$ (we assume $\mathcal{H}$ is closed under scale), we have that either $\langle h', f^*\rangle >0$ or $\langle -h', f^* \rangle >0$. Namely, we find at least one hypothesis $h$ such that $\langle h, f^*\rangle >0$ and $\|h\| = 1$.  Hence if we pick $\tilde{h} = \arg\max_{h\in\mathcal{H}, \|h\|=1} \langle h, f^*/\|f^*\| \rangle$, we must have $\langle  \tilde{h}, f^*/\|f^*\|\rangle = \epsilon >0$. Namely we can find a hypothesis $h^*\in\mathcal{H}$, which is $\epsilon\|f^*\|\tilde{h}$, such that there is non-zero $\gamma\in(0,1]$:
\begin{align}
\|h^* - f^* \|^2\leq (1-\gamma)\|f^*\|^2.
\end{align}

To show that we can extend this $\gamma$ to the finite sample case, we are going to use Hoeffding inequality to relate the norm $\|\cdot\|$ to its finite sample approximation. 

Applying Hoeffding inequality, we get with probability at least $1-\delta/2$,
\begin{align}
\label{eq:high_prob_1}
|\frac{1}{T}\sum_{t=1}^T \|y_t\|^2 - \langle f^*, f^*\rangle| \leq O\big(\sqrt{\frac{F^2}{T}\ln(4/\delta)}\Big),
\end{align} where based on assumption that $f^*(\cdot)$ is bounded as $\|f^*(\cdot)\|\leq F$. Similarly, we have with probability at least $1-\delta/2$:
\begin{align}
|\frac{1}{T}\sum_{t=1}^T \|h^*(x_t) - f^*(x_t)\|^2 - \|h^* - f^* \|^2| \leq O\Big(\sqrt{\frac{F^2}{T}\ln(4/\delta)}\Big),
\end{align}
Apply union bound for the above two high probability statements, we get with probability at least $1-\delta$,
\begin{align}
&|\frac{1}{T}\sum_{t=1}^T y_t^2 - \langle f^*, f^*\rangle| \leq O\Big(\sqrt{\frac{F^2}{T}\ln(4/\delta)}\Big), \;\;\; and,\nonumber\\
&|\frac{1}{T}\sum_{t=1}^T (h^*(x_t) - f^*(x_t))^2 - \|h^* - f^* \|| \leq O\Big(\sqrt{\frac{F^2}{T}\ln(4/\delta)}\Big). 
\end{align} Now to prove the theorem, we proceed as follows:
\begin{align}
&\frac{1}{T}\sum_{t=1}^T \|h^*(x_t) - f^*(x_t)\|^2 \nonumber\\
&\leq  \|h^* - f^* \| + O\Big(\sqrt{\frac{F^2}{T}\ln(4/\delta)}\Big)\nonumber\\
&\leq (1-\gamma)\|f^*\|^2+O\Big(\sqrt{\frac{F^2}{T}\ln(4/\delta)}\Big)\nonumber\\
&\leq (1-\gamma)\frac{1}{T}\sum_{t=1}^T y_t^2 + (1-\gamma)O\Big(\sqrt{\frac{F^2}{T}\ln(4/\delta)}\Big) + O\Big(\sqrt{\frac{F^2}{T}\ln(4/\delta)}\Big).
\end{align} Hence we get with probability at least $1-\delta$:
\begin{align}
\sum_{t=1}^T \|h^*(x_t) - f^*(x_t)\|^2 \leq \sum_{t=1}^T \|y_t\|^2 + (2-\gamma)O\Big(\sqrt{ T\ln(1/\delta)}\Big).
\end{align} Set $R(T) = R_{\mathcal{A}}(T) + (2-\gamma)O\Big(\sqrt{ T\ln(1/\delta)}\Big)$, we prove the proposition. 
\end{proof}

\section{Proof of Theorem~\ref{them:smooth_strongly_convex}}
\label{sec:proof_smooth_strongly_convex}

An important property of $\lambda$-strong convexity that we will use later in the proof is that for any $x$ and $x^*=\arg\min_x l(x)$, we have:
\begin{align}
\label{eq:strong_convexity_property}
\|\nabla l(x)\|^2 \geq 2\lambda (l(x) - l(x^*)).
\end{align}We prove Eqn.~\ref{eq:strong_convexity_property} below.

From the $\lambda$-strong convexity of $l(x)$, we have:
\begin{align}
l(y) \geq l(x) + \nabla l(x)(y-x) + \frac{\lambda}{2}\|y - x\|^2.
\end{align} Replace $y$ by $x^*$ in the above equation, we have:
\begin{align}
&l(x^*) \geq l(x) + \nabla l(x)(x^* - x) + \frac{\lambda}{2}\|x^* - x\|^2 \nonumber \\
\Rightarrow & 2\lambda l(x^*) \geq 2\lambda l(x) + 2\lambda \nabla l(x)(x^* - x) + \lambda^2\|x^* - x\|^2 \nonumber \\
\Rightarrow & -2\lambda \nabla l(x)(x^* - x) - \lambda^2\|x^* - x\|^2 \geq 2\lambda (l(x) - l(x^*)) \nonumber \\
\Rightarrow &  \|\nabla l(x)\|^2 - \|\nabla l(x)\|^2-2\lambda \nabla l(x)(x^* - x) - \lambda^2\|x^* - x\|^2 \geq 2\lambda (l(x) - l(x^*)) \nonumber \\
\Rightarrow & \|\nabla l(x)\|^2 - \|\nabla l(x) + \lambda (x^* - x)\|^2 \geq 2\lambda(l(x) - l(x^*)) \nonumber \\
\Rightarrow & \|\nabla l(x)\|^2 \geq 2\lambda (l(x) - l(x^*)).
\end{align}

\subsection{Proofs for Lemma~\ref{lemma:from_weak_learning}}
\begin{proof}
Complete the square on the left hand side (LHS) of Eqn.~\ref{eq:weak_learner_eq}, we have:
\begin{align}
\sum \|y_t\|^2 - 2y_t^Th_t(x_t) + \|h_t(x_t)\|^2 \leq (1-\gamma)\sum_t \|y_t\|^2 + R(T).
\end{align} Now let us cancel the $\sum y_t^2$ from both side of the above inequality, we have:
\begin{align}
\sum -2y_t^T h_t(x_t) \leq \sum -2y_t^T h_t(x_t) + \|h_t(x_t)\|^2 \leq -\gamma\sum \|y_t\|^2 + R(T).
\end{align} Rearrange, we have:
\begin{align}
\sum 2y_t^T h_t(x_t) \geq \gamma\sum \|y_t\|^2 - R(T).
\end{align} 
\end{proof}

\subsection{Proof of Theorem~\ref{them:smooth_strongly_convex}}
We need another lemma for proving theorem~\ref{them:smooth_strongly_convex}:
\begin{lemma}
\label{lemma:helper_1}
For each weak learner $\mathcal{A}_i$, we have:
\begin{align}
\sum_t \|h_t^i(x_t)\|^2 \leq (4-2\gamma)\sum_t\|\nabla\ell_t(y_t^{i-1})\|^2 + 2R(T).
\end{align}
\end{lemma}
\begin{proof}[Proof of Lemma \ref{lemma:helper_1}] For $\sum_t (h_t^i(x_t))^2$, we have:
\begin{align}
&\sum_t\|h_t^i(x_t)\|^2 = \sum_t \|h_t^i(x_t) - \nabla\ell_t(y_t^{i-1}) + \nabla\ell_t(y_t^{i-1})\|^2 \nonumber \\
& \leq \sum_t \|h_t^i(x_t)-\nabla\ell_t(y_t^{i-1})\|^2 + \sum_t\|\nabla\ell_ty_t^{i-1}\|^2 + \sum_t 2(h_t^i(x_t) - \nabla\ell_t(y_t)^{i-1})^T\nabla\ell_t(y_t^{i-1}) \nonumber \\
& \leq  \sum_t 2\|h_t^i(x_t)-\nabla\ell_t(y_t^{i-1})\|^2 + \sum_t2 \|\nabla\ell_t(y_t^{i-1}\|^2 \nonumber \\
& \leq 2(1-\gamma)\sum_t\|\nabla\ell_t(y_t^{i-1}\|^2 +2R(T) + 2\sum_t\|\nabla\ell_t(y_t^{i-1}\|^2 \nonumber \\
& \;\;\;\;\;\;\;\; \textit{(By Weak Onling Learning Definition)} \nonumber \\
& \leq (4-2\gamma)\sum_t\|\nabla\ell_t(y_t^{i-1}\|^2 + 2R(T).
\end{align}
\end{proof}

\begin{proof}[Proof of Theorem~\ref{them:smooth_strongly_convex}]
For $1\leq i \leq N$, let us define $\Delta_i = \sum_{t=1}^T (\ell_t(y_t^i) - \ell_t(f^*(x_t)))$. Following similar proof strategy as shown in \citep{beygelzimer2015online}, we will link $\Delta_i$ to $\Delta_{i-1}$. For $\Delta_i$, we have:
\begin{align}
&\Delta_i = \sum_{t=1}^T (\ell_t(y_t^i) - \ell_t(f^*(x_t))) = \sum_t \ell_t(y_t^{i-1} - \eta h_t^i(x_t)) - \sum_t\ell_t(f^*(x_t)) \nonumber \\
& \leq  \sum_t \big[ \ell_t(y_t^{i-1}) - \eta\nabla\ell_t(y_t^{i-1})^T h_t^i(x_t) + \frac{\beta\eta^2}{2}\|h_t^i(x_t)\|^2    \big] - \sum_t\ell_t(f^*(x_t)) \nonumber \\
& \;\;\;\;\;\;\;\; \textit{(By $\beta$-smoothness of $\ell_t$)} \nonumber \\
& \leq \sum_t \big[ \ell_t(y_t^{i-1}) - \frac{\eta\gamma}{2}\|\nabla\ell_t(y_t^{i-1})\|^2 + \frac{\eta R(T)}{2}  + \frac{\beta\eta^2}{2}\|h_t^i(x_t)\|^2\big] - \sum_t\ell_t(f^*(x_t)) \nonumber \\
& \;\;\;\;\;\;\;\; \textit{(By Lemma~\ref{lemma:from_weak_learning})} \nonumber \\
&\leq \sum_t \big[ \ell_t(y_t^{i-1}) - \frac{\eta\gamma}{2}\|\nabla\ell_t(y_t^{i-1})\|^2 + \frac{\eta R(T)}{2}  + \beta\eta^2(2-\gamma)\|\nabla\ell_t(y_t^{i-1})\|^2 + \beta\eta^2R(T) - \ell_t(f^*(x_t))\big] \nonumber \\
& \;\;\;\;\;\;\;\; \textit{(By Lemma~\ref{lemma:helper_1})} \nonumber \\
& = \Delta_{i-1}- (\frac{\eta\gamma}{2} - \beta\eta^2(2-\gamma))\sum_t\|\nabla\ell_t(y_t^{i-1})\|^2 + (\frac{\eta}{2}+\beta\eta^2)R(T) \nonumber \\
&\leq \Delta_{i-1} - (\eta\gamma\lambda - \beta\eta^2\lambda(4-2\gamma))\sum_t\big(\ell_t(y_t^{i-1}) - \ell_t(f^*(x_t))\big) + (\frac{\eta}{2}+\beta\eta^2)R(T) \nonumber \\
& \;\;\;\;\;\;\;\; \textit{(By Eqn.~\ref{eq:strong_convexity_property})} \nonumber \\
& = \Delta_{i-1}\big[ 1 - (\eta\gamma\lambda - \beta\eta^2\lambda(4-2\gamma)) \big] + (\frac{\eta}{2}+\beta\eta^2)R(T)
\end{align}
Due to the setting of $\eta$, we know that $0<(1 - (\eta\gamma\lambda - \beta\eta^2\lambda(4-2\gamma))) <1$. For notation simplicity, let us first define $C = 1 - (\eta\gamma\lambda - \beta\eta^2\lambda(4-2\gamma))$. Starting from $\Delta_0$, keep applying the relationship between $\Delta_i$ and $\Delta_{i-1}$ $N$ times, we have:
\begin{align}
&\Delta_N = C^N\Delta_0 + (\frac{\eta}{2}+\beta\eta^2)R(T)\sum_{i=1}^N C^{i-1}\nonumber \\
&= C^N\Delta_0 + (\frac{\eta}{2}+\beta\eta^2)R(T) \frac{1-C^N}{1-C}\nonumber\\
&\leq C^N\Delta_0 + (\frac{\eta}{2}+\beta\eta^2)R(T)\frac{1}{1-C}. \nonumber
\end{align}  Now divide both sides by $T$, and take $T$ to infinity, we have:
\begin{align}
\frac{1}{T}\Delta_N = C^N\frac{1}{T}\Delta_0 \leq  C^N 2B,
\end{align} where we simply assume that $\ell_t(y) \in [-B,B], B\in\mathcal{R}^+$ for any $t$ and $y$.
Now let us go back to $C$, to minimize $C$, we can take the derivative of $C$ with respect to $\eta$, set it to zero and solve for $\eta$, we will have:
\begin{align}
\eta = \frac{\gamma}{\beta(8-4\gamma)}. 
\end{align}
Substitute this $\eta$ back to $C$, we have:
\begin{align}
C = 1 - \frac{\gamma^2\lambda}{\beta(16-8\gamma)} \geq 1 - \frac{\lambda}{8\beta} \geq 1- \frac{1}{8} = \frac{7}{8}.
\end{align}
Hence, we can see that there exist a $\eta = \frac{\gamma}{\beta(8-4\gamma)}$, such that:
\begin{align}
\label{eq:before_exp}
\frac{1}{T}\Delta_N \leq 2B (1 - \frac{\gamma^2\lambda}{\beta(16-8\gamma)})^N 
    \leq 2B (1 - \frac{\gamma^2\lambda}{16\beta})^N.
\end{align} Hence we prove the first part of the theorem regarding the regret. For the second part of the theorem where $\ell_t$ and $x_t$ are i.i.d sampled from a fixed distribution, we proceed as follows. 

Let us take expectation on both sides of the inequality~\ref{eq:before_exp}. The left hand side of inequality~\ref{eq:before_exp} becomes:
\begin{align}
&\frac{1}{T}\mathbb{E}\Delta_N = \mathbb{E}\frac{1}{T}\big[\sum_{t=1}^T (\ell_t(y_t^N) - \ell_t(f^*(x_t)))\big] =\frac{1}{T} \mathbb{E}\big[\sum_{t=1}^T\ell_t(-\mu\sum_{i=1}^N h_t^{i}(x_t))\big] - \frac{1}{T}\mathbb{E}_{(\ell_t,x_t)\sim D}[\ell_t(f^*(x_t))] \nonumber\\
& = \frac{1}{T}\sum_{i=1}^T\mathbb{E}_t\big[\ell_t(-\mu\sum_{i=1}^N h_t^i(x_t))\big] - \mathbb{E}_{(\ell,x)\sim D}\ell(f^*(x)),
\end{align} where the expectation is taken over the randomness of $x_t$ and $\ell_t$. Note that $h_t^i$ only depends on $x_1,\ell_1,...,x_{t-1},\ell_{t-1}$. We also define $\mathbb{E}_t$ as the expectation over the randomness of $x_t$ and $\ell_t$ at step $t$ conditioned on $x_{1}, \ell_1, ..., x_{t-1},\ell_{t-1}$. Since $\ell_t, x_t$ are sampled i.i.d from $D$, we can simply write $\mathbb{E}_t[\ell_t(-\mu\sum_{i=1}^N h_t^i(x_t))]$ as $\mathbb{E}_t[\ell(-\mu\sum_{i=1}^N h_t^i(x))]$. Now the above inequality can be simplied as:
\begin{align}
&\frac{1}{T}\mathbb{E}\Delta_N = \frac{1}{T}\sum_{t=1}^T \mathbb{E}_t[\ell(-\mu\sum_{i=1}^N h_t^i(x))] - \mathbb{E}_{(\ell,x)\sim D}\ell(f^*(x)) \nonumber\\
&\geq \mathbb{E}\big[\ell(-\mu\sum_{i=1}^N\frac{1}{T}\sum_{t=1}^T h_t^i(x))\big] - \mathbb{E}_{(\ell,x)\sim D}\ell(f^*(x)) \nonumber\\
& = \mathbb{E}\big[\ell(-\mu\sum_{i=1}^N \bar{h}_i(x))\big] - \mathbb{E}_{(\ell,x)\sim D}\ell(f^*(x))
\end{align}
Now use the fact that $1/T\mathbb{E}\Delta_N \leq 2B(1-\frac{\gamma^2\lambda}{16\beta})^N$, we prove the theorem.
\end{proof}

\section{Proof of Theorem~\ref{them:non_smooth_strongly_convex}}
\label{sec:proof_non_smooth_strongly_convex}

\begin{lemma}
\label{lem:residual_shrink}
In Alg.~\ref{alg:online_gradient_boost_non_smooth_residual}, 
if we assume the 2-norm of gradients of the loss w.r.t. partial sums by $G$
(i.e., $\Vert \nabla _t^i \Vert = \Vert \nabla \ell _t(y^{i-1}_{t})\Vert \leq G$), 
and assume that each weak learner $\mathcal{A}_i$ has regret $R(T) = o(T)$, then we 
there exists a constant 
$c = \frac{1- \gamma + \sqrt{1 - \gamma (1 - \frac{R(T)}{TG^2})}}{\gamma} 
< \frac{2}{\gamma} - 1$ 
such that
\begin{align}
    \sum _{t=1}^T \Vert \Delta _i^t \Vert^2 \leq c^2G^2 T 
    \quad \text{and} \quad 
    \sum_{t=1}^T \Vert h^t_i(x_t) \Vert^2 \leq (4-2\gamma)(1+c)^2G^2T + 2R(T) \leq 4c^2G^2T.
\end{align}
\end{lemma}
\begin{proof}
We prove the first inequality by induction on the weak learner index $i$. When $i = 0$, the claim is clearly true since $\Delta _0  ^t = 0 $ for all $t$. Now we assume the claim is true for some $i \geq 0$, and prove it for $i+1$. We first note that by the inequality 
    $\frac{1}{T} \sum _{t=1}^T a_t  \leq \sqrt{ \frac{\sum _t a_t^2}{T} }$ for all sequence 
    $\{a_t \}_t$, we have 
\begin{align}
 &   \frac{1}{T} ( \sum _t \Vert \Delta _i^t \Vert )^2 
            \leq \sum _t \Vert \Delta _i^t \Vert^2 \leq c^2 G^2 T \\
\Rightarrow & 
    ( \sum _t \Vert \Delta _i^t \Vert ) ^2 \leq c^2G^2 T^2 \\
\Rightarrow & 
    \sum _t \Vert \Delta _i^t \Vert  \leq cGT
\end{align}
Then by the assumption that weak learner $\mathcal{A}_i$ has an edge $\gamma$ with 
regret $R(T)$, we have from step 14 of Alg.~\ref{alg:online_gradient_boost_non_smooth_residual}: 
\begin{align}
    \sum _t \Vert \Delta _{i+1}^t \Vert ^2 
    &= \sum _t \Vert \Delta _{i}^t + \nabla _{i+1}^t - h^t_{i+1}(x_t) \Vert ^2
        \leq (1-\gamma) \sum _t \Vert \Delta _{i}^t + \nabla _{i+1}^t \Vert^2 + R(T) \\
    &\leq (1-\gamma) \sum _t \left( \Vert \Delta_i^t \Vert + G \right)^2  + R(T) \\
    &\leq (1-\gamma) \left(\sum _t \Vert \Delta_i^t \Vert^2 + 
        2G \sum_t \Vert \Delta_i^t \Vert + G^2T \right) + R(T) \\
    &\leq (1-\gamma)(1+c)^2G^2T + R(T) \\
    &= c^2G^2T
\end{align}
We have the last equality because $c$ is chosen as the positive root of the 
quadratic equation: $\gamma c^2 + (2\gamma - 2) c + (\gamma -1 - \frac{R(T)}{TG^2}) = 0$, which is equivalent to $c^2G^2T = (1-\gamma)(c+1)^2G^2T + R(T)$. 

The second inequality of the lemma can be derived from a similar argument of Lemma~\ref{lemma:helper_1} by expanding 
$\Vert \left(\Delta_{i-1}^t + \nabla_{i}^t - h^t_i(x_t) \right) - \left(\Delta_{i-1}^t + \nabla_{i}^t\right) \Vert^2$ and then applying edge assumption. 
\end{proof}

We now use the above lemma to prove the performance guarantee of Alg.~\ref{alg:online_gradient_boost_non_smooth_residual} as follows.
\begin{proof}[Proof of Theorem~\ref{them:non_smooth_strongly_convex}]
We first define the intermediate predictors as: 
$f^t_0(x) := h_0(x)$,
\mbox{$\hat{f}^t_{i}(x) := f^{t-1}(x) - \eta_i h^t_{i}(x)$}, and \mbox{$f^t_i(x) := P(\hat{f}^t_i(x))$}. Then for all $i=1,..., N$ we have:

\begin{align}
&\Vert f^t_{i}(x_t) - f^*(x_t) \Vert^2 
    \leq \Vert \hat{f}^t_{i}(x_t) - f^*(x_t) \Vert^2 
    = \Vert f^t_{i-1}(x_t) - \eta_i h^t_{i}(x_t) - f^*(x_t) \Vert^2\\
&= \Vert f^t_{i-1}(x_t) - f^*(x_t) \Vert^2
    + \eta_i^2 \Vert h^t_i(x_t) \Vert^2 
    - 2\eta_i \innerprod{f^t_{i-1}(x_t) - f^*(x_t)}
                        {h^t_i(x_t) - \Delta_{i-1}^t -\nabla_i^t} \nonumber \\
    &- 2\eta_i \innerprod{f^t_{i-1}(x_t) - f^*(x_t)}
                        {\Delta_{i-1}^t + \nabla_i^t} 
\intertext{Rearanging terms we have:}
&\innerprod{f^*(x_t) - f^t_{i-1}(x_t)}{\nabla_i^t}\\
\geq &  \frac{1}{2\eta_i} \Vert f^t_{i}(x_t) - f^*(x_t) \Vert^2 
        -\frac{1}{2\eta_i} \Vert f^t_{i-1}(x_t) - f^*(x_t) \Vert^2
        -\frac{\eta_i}{2} \Vert h^t_i(x_t) \Vert^2 \nonumber \\
     &  -\innerprod{ f^*(x_t) - f^t_{i-1}(x_t) }
                        {h^t_i(x_t) - \Delta_{i-1}^t -\nabla_i^t} 
        -\innerprod{f^*(x_t) - f^t_{i-1}(x_t)}{\Delta_{i-1}^t}\\
\intertext{Using $\lambda$-strongly convex of $\ell_t$ and applying the above
    equality and $\Delta_i^t = \Delta_{i-1}^t + \nabla_i^t - h^t_i(x_t)$, we have:}
&\ell_t (f^*(x_t)) 
\geq  \ell_t (f^t_{i-1}(x_t)) 
    + \innerprod{f^*(x_t) - f^t_{i-1}(x_t)}{\nabla _{i}^t} 
    + \frac{\lambda}{2} \Vert f^*(x_t) - f^t_{i-1}(x_t) \Vert^2 \\
\geq& \ell_t (f^t_{i-1}(x_t)) 
    + \frac{1}{2\eta_i} \Vert f^t_{i}(x_t) - f^*(x_t) \Vert^2 
    -\frac{1}{2\eta_i} \Vert f^t_{i-1}(x_t) - f^*(x_t) \Vert^2
    -\frac{\eta_i}{2} \Vert h^t_i(x_t) \Vert^2 
\nonumber \\
&   +\innerprod{ f^*(x_t) - f^t_{i-1}(x_t) }{\Delta_i^t} 
    -\innerprod{f^*(x_t) - f^t_{i-1}(x_t)}{\Delta_{i-1}^t}
    + \frac{\lambda}{2} \Vert f^*(x_t) - f^t_{i-1}(x_t) \Vert^2 \\
\intertext{Summing over $t=1,...,T$ and $i=1,...,N$ we have:}
&N\sum _{t=1}^T\ell_t (f^*(x_t)) \nonumber \\
\geq& \sum _{i=1}^N \sum _{t=1}^T \left[\ell_t (f^t_{i-1}(x_t)) 
    +\innerprod{f^*(x_t) - f^t_{i-1}(x_t)}{\nabla _{i}^t} 
    +\frac{\lambda}{2} \Vert f^*(x_t) - f^t_{i-1}(x_t) \Vert^2 \right]\\
=& \sum _{i=1}^N \sum _{t=1}^T\ell_t (f^t_{i-1}(x_t)) 
    -\sum _{i=1}^N \sum _{t=1}^T \frac{\eta_i}{2} \Vert h^t_i(x_t) \Vert^2 
\nonumber \\
&   + \sum _{i=1}^N \sum _{t=1}^T 
        \frac{1}{2\eta_i} \Vert f^t_{i}(x_t) - f^*(x_t) \Vert^2 
    - \sum _{i=1}^N \sum _{t=1}^T 
        (\frac{1}{2\eta_i} - \frac{\lambda}{2}) \Vert f^t_{i-1}(x_t) - f^*(x_t) \Vert^2
\nonumber \\
&   +\sum _{i=1}^N \sum _{t=1}^T 
        \innerprod{ f^*(x_t) - f^t_{i-1}(x_t) }{\Delta_{i}^t} 
    -\sum _{i=1}^N \sum _{t=1}^T 
        \innerprod{f^*(x_t) - f^t_{i-1}(x_t)}{\Delta_{i-1}^t} \\
=& \sum _{i=1}^N \sum _{t=1}^T\ell_t (f^t_{i-1}(x_t)) 
    -\sum _{i=1}^N \sum _{t=1}^T \frac{\eta_i}{2} \Vert h^t_i(x_t) \Vert^2 
\nonumber \\
&   + \sum _{i=1}^N \sum _{t=1}^T 
        \frac{1}{2\eta_i} \Vert f^t_{i}(x_t) - f^*(x_t) \Vert^2 
    - \sum _{i=0}^{N-1} \sum _{t=1}^T 
        (\frac{1}{2\eta_{i+1}} - \frac{\lambda}{2}) \Vert f^t_{i}(x_t) - f^*(x_t) \Vert^2
\nonumber \\
&   + \sum _{i=1}^N \sum _{t=1}^T 
        \innerprod{ f^*(x_t) - f^t_{i-1}(x_t) }{\Delta_{i}^t} 
    - \sum _{i=1}^{N-1} \sum _{t=1}^T 
        \innerprod{f^*(x_t) - (f^t_{i-1}(x_t) - \eta_i h^t_i(x_t))}{\Delta_{i}^t} 
\nonumber \\
&   - \sum_{t=1}^T \innerprod{f^*(x_t) - f^t_{0}(x_t)}{\Delta_{0}^t} 
    \;\;\;\;\;\;\;\;\;\text{   (We switched index and apply $\Delta_0^t = 0$ next.)}\\
=& \sum _{i=1}^N \sum _{t=1}^T\ell_t (f^t_{i-1}(x_t)) 
    -\sum _{i=1}^N \sum _{t=1}^T \frac{\eta_i}{2} \Vert h^t_i(x_t) \Vert^2 
    - \sum_{i=1}^{N-1} \sum_{t=1}^T 
        \innerprod{\eta_i h^t_i(x_t)}{\Delta_{i}^t} 
\nonumber \\
&   + \sum _{i=1}^{N-1} \sum _{t=1}^T 
        \frac{1}{2} \Vert f^t_{i}(x_t) - f^*(x_t) \Vert^2
            (\frac{1}{\eta_{i}} - \frac{1}{\eta_{i+1}} + \lambda)
    - \sum_{t=1}^T (\frac{1}{2\eta_1} - \frac{\lambda}{2}) \Vert f^t_{0}(x_t) - f^*(x_t) \Vert^2
\nonumber \\
&   + \sum_{t=1}^T 
        \left[\innerprod{f^*(x_t) - f^t_{N-1}(x_t)}{\Delta_{N}^t} 
        + \frac{1}{2\eta_N} \Vert f_{N-1}^t(x_t) - \eta_N h_N^t(x_t) - f^*(x_t) \Vert^2 \right] \\
&\;\;\;\;\;\;\;\;\; \text{(We next apply $\eta_i = \frac{1}{\lambda i}$ and complete the squares for the last sum.)} \nonumber \\
=& \sum _{i=1}^N \sum _{t=1}^T\ell_t (f^t_{i-1}(x_t)) 
    -\sum _{i=1}^N \sum _{t=1}^T \frac{\eta_i}{2} \Vert h^t_i(x_t) \Vert^2 
    - \sum_{i=1}^{N-1} \sum_{t=1}^T 
        \innerprod{\eta_i h^t_i(x_t)}{\Delta_{i}^t} 
\nonumber \\
&   + \frac{1}{2\eta_N} \sum_{t=1}^T 
        \Vert \left( f_{N-1}^t(x_t) - f^*(x_t)\right)  + \eta_N(\Delta_N^t - h_N^t(x_t))\Vert^2 
\nonumber \\
&   - \frac{\eta_N}{2}  \sum_{t=1}^T \left( \Vert \Delta_N^t - h_N^t(x_t)\Vert^2 - 
        \Vert h_N^t(x_t)\Vert^2  \right)
        \\
&\;\;\;\;\;\;\;\;
\text{(We next drop the completed square, and apply Cauchy-Schwarz)} \nonumber \\
\geq& \sum _{i=1}^N \sum _{t=1}^T\ell_t (f^t_{i-1}(x_t)) 
    -\sum _{i=1}^N \sum _{t=1}^T \frac{\eta_i}{2} \Vert h^t_i(x_t) \Vert^2 
    - \sum_{i=1}^{N} \eta_i \sum_{t=1}^T 
        \Vert h^t_i(x_t) \Vert \Vert \Delta_{i}^t \Vert
    - \frac{\eta_N}{2}  \sum_{t=1}^T \Vert \Delta_N^t \Vert^2     \\
&\;\;\;\;\;\;\;\;
\text{(We next apply Cauchy-Schwarz again.)} \nonumber \\
\geq& \sum _{i=1}^N \sum _{t=1}^T\ell_t (f^t_{i-1}(x_t)) 
    -\sum _{i=1}^N  \frac{\eta_i}{2}  \sum _{t=1}^T\Vert h^t_i(x_t) \Vert^2 
    - \frac{\eta_N}{2}  \sum_{t=1}^T \Vert \Delta_N^t \Vert^2 \nonumber \\
    &- \sum_{i=1}^{N}  
        \eta_i \sqrt{\sum_{t=1}^T \Vert h^t_i(x_t)  \Vert^2 
                     \sum_{t=1}^T \Vert \Delta_{i}^t\Vert^2}
\end{align}
Now we apply Lemma~\ref{lem:residual_shrink} and replace the remaining $\eta_i = \frac{1}{\lambda i}$. Using $\sum _{i=1}^N \frac{1}{i} \leq 1 + \ln N$, we have:
\begin{align}
&N\sum _{t=1}^T\ell_t (f^*(x_t))  \nonumber \\
\geq& \sum _{i=1}^N \sum _{t=1}^T\ell_t (f^t_{i-1}(x_t)) 
    -\sum _{i=1}^N \frac{1}{2i\lambda } 4c^2G^2T
    - \frac{1}{2N\lambda }  c^2G^2T
    - \sum_{i=1}^{N}  
        \frac{1}{i\lambda} 2c^2G^2T \\
\geq& \sum _{i=1}^N \sum _{t=1}^T\ell_t (f^t_{i-1}(x_t)) 
    -\frac{4c^2G^2T}{\lambda} (1 + \ln N)
    - \frac{c^2G^2T}{2N\lambda }  
\end{align}
Dividing both sides by $NT$ and rearrange terms, we get:
\begin{align}
\frac{1}{TN}\sum_{i=1}^N\sum_{t=1}^T\big[\ell_t(y_t^i) - \ell_t(f^*(x_t))\big] \leq \frac{4c^2G^2}{N\lambda}(1+\ln N) + \frac{c^2G^2}{2N^2\lambda}.  \nonumber
\end{align} Using Jensen's inequality for the LHS of the above inequality, we get:
\begin{align}
\label{eq:first_part}
\frac{1}{T}\sum_{t=1}^T\ell_t(\frac{1}{N}\sum_{i=1}^N y_t^i) - \ell_t(f^*(x_t)) \leq \frac{4c^2G^2}{N\lambda}(1+\ln N) + \frac{c^2G^2}{2N^2\lambda},  \nonumber
\end{align} which proves the first part of the theorem. 

For stochastic setting, we can prove it by using similar proof techniques (e.g., take expectation on both sides of Eqn.~\ref{eq:first_part} and use Jensen inequality) that we used for proving  theorem~\ref{them:smooth_strongly_convex}. 
\end{proof}

\section{Counter Example for
Alg.~\ref{alg:online_gradient_boost}}
\label{sec:counter_example}
In this section, we provide an counter example where we show that Alg.~\ref{alg:online_gradient_boost} cannot guarantee to work for non-smooth loss. We set $y\in\mathbb{R}^2$, and design a loss function $\ell_t(y) = 2|y_{[1]}|+|y_{[2]}|$, where $y_{[i]}$ stands for the i'th entry of the vector $y$, for all time step $t$. The subgradient of this non-smooth loss is $[2,1]^T$, or $[2,-1]^T$, or $[-2,1]^T$, or $[-2,-1]^T$, depending on the position of $y$. We restricted the weak hypothesis class $\mathcal{H}$ to consist of only two types of hypothesis: hypothesis $h(x) = [\alpha, 0]^T$, or hypothesis $h(x) = [0,\alpha]^T$, where $\alpha \in [-2,2]$. We can show that given a sequence of training examples $\{(x_{\tau}, g_{\tau})\}_{\tau=1}^t$, where $g_t$ is the one of the gradient from the total four possible subgradient of $\ell_t$, the hypothesis that minimizes the accumulated square loss $\sum_{\tau=1}^t (h(x_{\tau}) - g_{\tau})^2$ is going to be the type of $h(x) = [\alpha, 0]^T$.

Now we consider using Follow the Leader (FTL) as a no-regret online learning algorithm for each weak learner. Based on the above analysis, we know that no matter what the sequence of training examples each weak learner has received as far, the weak leaners always choose the hypothesis with type $h(x) = [\alpha,0]^T$ from $\mathcal{H}$. So, for every time step $t$, if we initialize $y_t^0 = [a,b]^T$, where $a>0$ and $b>0$, then the output $y_t^N$ (computed from Line 8 in Alg.1) always have the form of $y_t^N = [\eta, b]$, where $\eta\in\mathbb{R}$. Namely, all weak learners' prediction only moves $y_t$ horizontally and it will never be moved vertically. But note that the optimal solution is located at $[0,0]^T$. Since for all $t$, $y_{t_{[2]}}^N$ is also $b$ constant away from $0$, the total regret accumulates linearly as $bT$, regardless of how many weak learners we have.

\section{Details of  Implementation}
\label{sec:implementation}
\subsection{Binary Classification}
For binary classification, following \citep{friedman2001greedy}, let us define feature $x\in\mathbb{R}^n$, label $u\in\{-1,1\}$. With $x_t$ and $u_t$, the loss function $\ell_t$ is defined as:
\begin{align}
\ell_t(y) = \ln(1 + \exp(-u_t y)) + \lambda y^2.
\end{align} where $y\in\mathbb{R}$. In this setting, we have $\mathcal{H}: \mathbb{R}^n\rightarrow \mathbb{R}$. The regularization is to avoid overfitting: we can set $y = \infty*sign(u_t)$ to make the loss close to zero. 

The loss function $\ell_t(y)$ is twice differentiable with respect to $y$, and the second derivative is:
\begin{align}
\nabla^2\ell_t(y) = \frac{\exp(u_t y)}{(1+\exp(u_t y))^2}
\end{align}
Note that we have:
\begin{align}
\nabla^2\ell_t(y) \leq \frac{1}{1/\exp(u_t y) + 2 + \exp(u_t y)} \leq \frac{1}{4}.
\end{align} Hence, $\ell_t(y)$ is $1/4$-smooth.

Under the assumption that the output from hypothesis from $\mathcal{H}$ is bounded as $|y| \leq Y\in\mathbb{R}^+$, we also have:
\begin{align}
\nabla^2\ell_t(y) \geq \frac{1}{2+2\exp(Y)}
\end{align}
Hence, with boundness assumption, we can see that $\ell_t(y)$ is $1/(2+2\exp(Y))$-strongly convex and $(1/4)$-smooth. 

The another loss we tried is the hinge loss:
\begin{align}
\ell_t(y) = \max(0, 1 - u_t y) + \lambda y^2.
\end{align}With the regularization, the loss $\ell_t(y)$ is still strongly convex, but no longer smooth.

\subsection{Multi-class Classification}
Follow the settings in \citep{friedman2001greedy}, for multi-class classification problem, let us define feature $x\in\mathbb{R}^n$, and label information $u\in\mathbb{R}^k$, as a one-hot representation, where $u[i] = 1$ ($u[i]$ is the i-th element of $u$), if the example is labelled by $i$, and $u[i] = 0$ otherwise. The loss function $\ell_t$ is defined as:
\begin{align}
\ell_t(y) = -\sum_{i=1}^k u_t[i]\ln \frac{\exp(y[i])}{\sum_{j=1}^k\exp(y[j])},
\end{align} where $y\in\mathbb{R}^k$. In this setting, we let  weak learner $i$ pick hypothesis $h$ from $\mathcal{H}$ that takes feature $x_t$ as input, and output $\hat{y}_i\in\mathbb{R}^k$. The online boosting algorithm then linearly combines the weak learners' prediction to predict $y$.

\end{document}